\documentclass[lettersize,journal]{IEEEtran}
\usepackage{amsmath,amssymb,amsfonts}
\usepackage{amsthm}
\usepackage{bm}
\usepackage{algorithmic}
\usepackage{algorithm}
\usepackage{array}
\usepackage[caption=false,font=normalsize,labelfont=sf,textfont=sf]{subfig}
\usepackage{textcomp}
\usepackage{stfloats}
\usepackage{url}
\usepackage{booktabs}
\usepackage{verbatim}
\usepackage{graphicx}
\usepackage{tabularx}
\usepackage{cite}
\usepackage[colorlinks,
            linkcolor=blue,       %%修改此处为你想要的颜色
            anchorcolor=blue,  %%修改此处为你想要的颜色
            citecolor=blue,        %%修改此处为你想要的颜色，例如修改blue为red
            ]{hyperref}
\theoremstyle{definition}
\newtheorem{assumption}{Assumption}

\newtheorem{theorem}{Theorem}
\newtheorem{lemma}{Lemma}
\theoremstyle{remark}
\newtheorem*{remark}{Remark}

\usepackage{graphicx}
\usepackage{multirow}
\allowdisplaybreaks[4]
\begin{document}

\title{Multi-Agent Conditional Diffusion Planning for Decentralized Wireless Resource Allocation}

\author{Kechen Meng, Sinuo Zhang, Rongpeng Li, Yansha Deng, Chan Wang, and Zhifeng Zhao
    \thanks{K. Meng, S. Zhang, R. Li and C. Wang are with the College of Information Science and Electronic Engineering, Zhejiang University (email: \{mengkechen, 22431100, lirongpeng, 0617464\}@zju.edu.cn).}
    \thanks{Y. Deng is with the Department of Engineering, King’s College London, London WC2R 2LS, U.K. (e-mail: {yansha.deng}@kcl.ac.uk). }
    \thanks{Z. Zhao is with Zhejiang Lab as well as the College of Information Science and Electronic Engineering, Zhejiang University (email: zhaozf@zhejianglab.org).}
}



\maketitle

\begin{abstract}
In wireless communication systems, efficient and adaptive resource allocation is essential for enhancing overall Quality of Service (QoS).
Compared with conventional Model-Free Reinforcement Learning (MFRL), Model-Based RL (MBRL) first learns a generative world model for subsequent planning, enabling more stable training through the reuse of historical experience. However, deploying MBRL in large-scale wireless networks remains challenging due to high-dimensional stochastic dynamics, strong inter-agent cooperation, and stringent communication constraints. 
To overcome these challenges, we propose a Multi-Agent Conditional Diffusion Model Planner (MA-CDMP) for decentralized communication resource management. Following the Decentralized Training with Decentralized Execution (DTDE) paradigm, MA-CDMP models each communication node as an autonomous agent and employs Diffusion Models (DMs) to capture environment dynamics. An inverse dynamics model is further incorporated to guide action generation, improving sample efficiency and policy scalability.
Moreover, to efficiently approximate large-scale agent interactions, we introduce a Mean Field (MF) mechanism to assist the classifier in conditional DMs. This mitigates inter-agent non-stationarity and enhances cooperation with minimal communication overhead. We further establish an upper bound on the distributional approximation error induced by MF-based diffusion generation, guaranteeing convergence stability and reliable modeling of multi-agent stochastic dynamics. Extensive experiments demonstrate that MA‑CDMP consistently outperforms existing MARL baselines in terms of average reward and QoS metrics, showcasing its scalability and practicality for real-world wireless network optimization.
\end{abstract}


\section{Introduction}

\IEEEPARstart{W}{ireless} communication networks have been extensively deployed due to their flexibility, low cost, and robustness\cite{alsabah20216g, kadhim2023enhancement}. Within these networks, efficient resource allocation is vital for optimizing spectrum and energy utilization\cite{aboagye2024multi}, yet increasing network complexity and localized observations limit system stability and scalability. Traditional rule-based schemes \cite{7148429, jabandzic2021dynamic} often fail under dynamic or high-load conditions, motivating the use of Reinforcement Learning (RL) for adaptive resource management \cite{wi2020delay, Chilukuri2021Deadline}. Nonetheless, widely adopted Model-Free RL (MFRL) methods might suffer from unstable training and low sample efficiency, due to their reliance on costly online interactions with the wireless environment. To address this, we adopt the Model-Based RL (MBRL) paradigm\cite{moerland2023model}, where a world model is learned from historical logs to simulate potential outcomes and generate synthetic trajectories, enabling planning-based decision-making with improved long-term reasoning\cite{huang2020model,you2019advanced}. 

Despite earlier efforts like model ensembling \cite{chua2018deep}, advanced encoding \cite{kingma2013auto} and adversarial learning \cite{luo2023reward,goodfellow2014generative}, building an accurate generative world model in complex and dynamic wireless networks remains a major challenge \cite{wu_highfidelity_2025}. %Recent progress in deep generative modeling offers valuable insights into realizing reliable planning. 
Fortunately, compared with Energy-Based Models (EBMs)~\cite{du2019implicit}, Variational Autoencoders (VAEs)~\cite{kingma2013auto}, and Generative Adversarial Networks (GANs)~\cite{goodfellow2014generative}, Diffusion Models (DMs) demonstrate superior sample quality and training stability through an iterative noise-to-data denoising process, representing the forefront of advances in generative learning \cite{li2022diffusion,rombach2022high}. However, existing works that combine DMs with RL mainly adopt a model-free pipeline \cite{li2023beyond, 10736570}, where DMs generate actions directly. %The integration of DMs with MBRL thus remains an open and promising direction for further exploration.
Although \cite{meng_conditional_2025a} demonstrates the superiority of diffusion-based MBRL for wireless resource allocation, this framework still operates under a centralized paradigm, leading to inefficient resource utilization due to heterogeneous service demands and the dynamic arrival and departure of services in wireless networks. In response, Multi-Agent RL (MARL) has emerged to further enhance adaptability by enabling decentralized decision-making\cite{sohaib2021dynamic, miuccio2024learning}. Existing MARL frameworks for wireless resource allocation can be broadly categorized into Centralized Training with Centralized Execution (CTCE), Centralized Training with Decentralized Execution (CTDE), and Decentralized Training with Decentralized Execution (DTDE) paradigms \cite{hwang2022decentralized}. CTCE assumes centralized information during both training and execution, which eases coordination but suffers from poor scalability and heavy signaling overhead. CTDE relaxes this requirement by allowing decentralized execution, yet it still depends on global states, joint actions, or global rewards during training, leading to scalability and privacy issues as the number of agents grows \cite{miuccio2024learning, xu2023distributed}. In contrast, DTDE\cite{chen2022sample, li2024multi} learns fully decentralized policies with local collaboration, offering better support for large-scale wireless communication networks. 

However, DTDE's decentralized structure limits global information sharing, which can result in non-stationarity in multi-agent learning\cite{hwang2022decentralized}. 
Therefore, the design of communication and fusion mechanisms is
essential yet challenging \cite{qu2022scalable}. A common approach is the consensus-based method\cite{qu2019value, chen2022sample}, which propagates and aggregates information through parameterized functions. While its convergence can be theoretically guaranteed under linear function approximation, it becomes unreliable when extended to neural networks. Another research direction employs Graph Neural Networks (GNNs) to model multi-agent interactions and facilitate cooperation through automated relational learning\cite{wang2022learning, blumenkamp2021emergence, nayak2023scalable}. Although GNN-based methods can capture complex dependencies, they often suffer from high computational costs and heavy data requirements. Compared to these consensus-based or graph-based methods, Mean Field (MF) \cite{yang2018mean}, which allows local information exchange through approximating all other neighbors' dynamics, offers both theoretical performance guarantees and computational efficiency.

In this paper, targeted at addressing the decentralized resource allocation mission in wireless communication networks, we propose the Multi-Agent Conditional Diffusion Model Planner (MA-CDMP). Following the DTDE framework for MARL training and execution, we model each network node as an independent agent. For each agent, a diffusion noise model is first employed to model environmental dynamics. Subsequently, an inverse dynamics model maps consecutive observation transitions to corresponding actions, facilitating the derivation of a dynamic resource allocation policy. Prominently, to promote cooperation and efficient communication, we integrate the MF mechanism with the classifier in DM to generate trajectories of high cumulative rewards, where the diffusion-based Stochastic Differential Equation (SDE) models each agent’s local observation trajectory together with the averaged observations of its 1-hop neighbors. We theoretically prove that, under MF approximation, the distribution error of diffusion-generated data admits a formal upper bound. It is worth noting that, although related work~\cite{yuan2025integrating} also combines MF and DMs, their design relies on communicating the global density distribution across all agents, whereas our framework focuses on exchanging averaged local observations within 1-hop neighborhoods and leverages a classifier-guided conditional diffusion model for task-specific trajectory generation. Compared with existing literature, the main contributions of this paper can be summarized as follows:
\begin{itemize}
\item We propose the MA-CDMP algorithm under the DTDE framework to achieve load-aware resource allocation in wireless communication networks, with each node modeled as an agent. %Due to the complexity of modeling discrete action sequences, a conditional DM captures observation transitions and guides planning toward high-reward trajectories following an offline MBRL paradigm, while an inverse dynamics model predicts actions based on consecutive observations. Additionally, 
%Besides employing a diffusion noise model for conditional decision generation following the MBRL paradigm, we integrate the MF mechanism, which approximates the influence of neighboring agents as an aggregated effect, with the classifier in DM. This design enables scalable coordination and efficient policy optimization, allowing agents to achieve cooperative behavior while maintaining low communication overhead.
We integrate DMs into the MA-MBRL paradigm for decision generation, leveraging a learned world model to perform conditional planning with improved sample efficiency. Moreover, an MF interaction mechanism is incorporated within the diffusion process to approximate the collective influence of neighboring agents as an aggregated effect, effectively mitigating inter-agent non-stationarity and enhancing coordination in large-scale systems.
\item We conduct a theoretical analysis that derives an upper bound on the distribution error of diffusion-generated data under MF communication. Specifically, we first bound the drift term fitting error in the diffusion SDE through MF approximation error analysis. Then, leveraging relative Fisher information\cite{sun2017relative}, we prove that the distribution approximation error of the SDE-based diffusion generative process also admits an explicit upper bound.
\item We establish formal performance guarantees for the MA-CDMP algorithm, demonstrating its accuracy and reliability under appropriate parameter configurations. The proposed framework is further evaluated through high-fidelity OPNET simulations, demonstrating its superior performance over existing MARL baselines, while extensive ablation studies further validate its robustness, scalability, and adaptability under diverse conditions.
\end{itemize} 

The remainder of the paper is organized as follows. Section~\ref{II} briefly introduces the related works. In Section~\ref{III}, we introduce the preliminaries and formulate the system model. The details of our proposed MA-CDMP algorithm are presented in Section~\ref{IV}, and Section~\ref{V} provides the theoretical analysis establishing its accuracy guarantees. In Section~\ref{VI}, we provide the results of extensive simulations and numerical analysis. Finally, the conclusion is summarized in Section~\ref{VII}.

For convenience, we list the major notations of this paper in Table \ref{tab1}.

\begin{table}[tbp]
\caption{Main parameters and notations used in this paper.}
\vspace{-0.6cm}
\begin{center}
\begin{tabular}{|c|m{6.5cm}|}
\hline
\toprule %[2pt]
\textbf{Notations} & \textbf{Definition} \\
\midrule %[2pt] 
$N$ & Number of nodes \\
\hline
$M$ & Number of time slots in a frame \\
\hline
$L$ & Number of channels \\
\hline
% $P_t$, $P_r$ & The transmit and received power \\
% \hline
% $G_t$, $G_r$ & The antenna gains at the transmitter and receiver \\
% \hline
% $f_c$ & The carrier frequency \\
% \hline
% $d$ & The transmission distance \\
% \hline
$K$ & Diffusion steps \\
\hline
$H$ & Planning horizon \\
\hline
$\bm{\epsilon}_{\theta}$ & Diffusion noise model \\
\hline
$\mathcal{J}_{\psi}$ & Classifier model \\
\hline
$f_{\phi}$ & Inverse dynamics model \\
\hline
$\mathcal{N}_i$ & Set of 1-hop neighbors of agent $i$ \\
\hline
$o^{(i)}_t$, $a^{(i)}_t$, $r^{(i)}_{t}$ & Local observation, action and reward of the agent $i$ at $t$ \\
\hline
$\overline{{o}}^{(i)}_t$ & Mean observation over agent $i$’s 1-hop neighbors at $t$ \\
\hline
$\bm{x}^{(i)}_k$ & The observation sequence of agent $i$ at diffusion step $k$ \\
\hline
$\overline{\bm{x}}^{(i)}_k$ & Mean observation sequence over agent $i$’s 1-hop neighbors at diffusion step $k$ \\
\hline
$\widetilde{\bm{x}}^{(i)}_k$ & Concatenated vector of $\bm{x}^{(i)}_k$ and $\overline{\bm{x}}^{(i)}_k$ \\
\hline
$\bm{y}^{(i)}$ & Conditional signal for agent $i$\\
\hline
${\rm gen}_{t}^{(i)}$ & Length of packet generation queue at $t$ \\
\hline
${\rm gen}_{\max,t}^{(i)}$ & Maximum length of packet generation queue within $t$ \\
\hline
$T_{t}^{(i)}$ & Length of packet transmission queue at $t$ \\
\hline
$T_{\max,t}^{(i)}$ & Maximum length of packet transmission queue within $t$ \\
\hline
$d_{t}^{(i)}$ & Time delay of node $i$ at $t$ \\
\hline
$\zeta$ & Conditional guidance scale \\
\hline
$L_{\mathcal{J}}, L_{\bm{\epsilon}}$ & Lipschitz constants of the classifier $\mathcal{J}$ and diffusion noise model $\bm{\epsilon}$ \\
\bottomrule %[2pt]  
\end{tabular}
\label{tab1}
\end{center}
\vspace{-0.5cm}
\end{table}

\section{Related Works}\label{II}
\subsection{Multi-Agent Reinforcement Learning for Resource Scheduling}
With the increasing complexity of wireless communication networks and the scarcity of spectrum resources, efficient scheduling is vital for managing channel access and optimizing network utilization. For instance, \cite{wi2020delay} proposes an actor–critic algorithm to optimize Time Division Multiple Access (TDMA) scheduling and minimize weighted end-to-end delay, while \cite{Chilukuri2021Deadline} encodes node features into state representations and applies heuristics across time slots to support long-term deadline adherence. However, these methods adopt the single-agent RL paradigm, where centralized control is constrained by limited scalability and the challenge of accurately capturing localized interactions. To address these issues, MARL has gained increasing attention. For instance, \cite{naderializadeh2021resource} proposes a multi-agent framework for joint user scheduling and power control, leveraging local observations and neighboring information to facilitate coordination among transmitters. Focusing on load-aware resource scheduling, \cite{9866568} develops a distributed framework for Multi-Frequency TDMA (MF-TDMA) wireless ad hoc networks, defining a network utility function parameterized by traffic loads. QLBT \cite{9681886} enhances QMIX \cite{rashid2018qmix} with an additional per-agent Q-value in the mixing network and integrates delay-to-last-successful-transmission (D2LT) into observations, enabling cooperative policies that prioritize agents with the longest delays. In \cite{10547350}, the authors address impractical power allocation policies in multi-carrier systems and enforce strict compliance with transmit power constraints. Nevertheless, these approaches follow the CTDE paradigm, which requires centralized information aggregation, leading to substantial communication overhead, privacy concerns, and limited expandability in large-scale deployments. To overcome these limitations, \cite{zhang2025multi} proposes a fully decentralized MARL framework for multi-cell spectrum and power allocation, relying solely on local information for decision-making. More recently, \cite{11481086} applies feasible MARL for time-frequency multiplexed industrial IoT MAC learning by adopting a Partially Observable Markov Game (POMG)-based formulation and a fully decentralized MAPPO implementation, where each device constructs local rewards from BS-broadcast resource-level feedback without relying on omniscient global rewards. However, these approaches remain \emph{model-free} and rely on heavy environmental interactions for policy learning. In contrast, our work develops an MBRL-based DTDE framework with MF communication, enabling decentralized agents to perform scalable planning and cooperation with significantly reduced communication cost.

\subsection{Diffusion-Based Resource Allocation in Communication Systems}
DMs have achieved remarkable success in various domains, including high-fidelity image synthesis\cite{rombach2022high}, natural language processing\cite{li2022diffusion}, and protein structure prediction\cite{watson2023protein}. Owing to their strong generative capabilities, recent studies have extended DMs to decision-making tasks in RL\cite{janner2022planning, ajay2022conditional}, modeling complex trajectory and action distributions. However, their application to wireless communication networks for resource allocation remains at an early stage, largely due to highly dynamic environments, strong interference, and stringent reliability requirements\cite{10515203}. Several emerging works have integrated DMs into wireless communication tasks targeting these difficulties. For example, \cite{10753523} proposes a DDPM-based framework to generate optimal block lengths conditioned on channel states. D3PG \cite{liu2024generative} combines DMs with DDPG \cite{lillicrap2015continuous} to adapt contention windows and aggregation frame lengths in Wi-Fi networks. Moreover, \cite{10838290} presents a hierarchical scheme, where a DM-enhanced soft actor–critic \cite{haarnoja2018soft} performs high-level spectrum allocation, while QMIX \cite{rashid2018qmix} handles low-level power control and resource assignment. Despite promising results, these methods primarily follow the MFRL paradigm\cite{ramirez2022model}, using DMs to directly parameterize policies. In addition, they are formulated under single-agent or centralized control settings, which are not directly applicable to decentralized MARL scenarios and often suffer from low sample efficiency and unstable convergence in dynamic wireless environments\cite{huang2020model}. In contrast, we employ the MBRL-based MARL framework, where DMs approximate the environment’s state transition distribution. This modeling strategy enables agents to plan actions without extensive real-world interaction, thereby improving sample efficiency, enhancing convergence stability, and supporting long-horizon planning in complex wireless network scenarios.

\subsection{Diffusion Models for Multi-Agent Reinforcement Learning}
Although DMs have demonstrated strong generative capabilities and been applied to decision-making in RL \cite{janner2022planning, ajay2022conditional}, their use in MARL tasks remains limited. Extending DMs to MARL is challenging due to non-stationarity, training instability, and the complexity of modeling inter-agent interactions\cite{chendeep}. To mitigate these limitations, DoF \cite{li2025dof} extends the Individual-Global-Max (IGM) \cite{rashid2018qmix} principle by introducing a noise factorization function to decompose a centralized diffusion model into agent-specific ones, and a data factorization function to capture dependencies among the generated data. MADiff \cite{zhu2024madiff} employs an attention-based diffusion model to learn complex coordination patterns, training a centralized diffusion process for joint trajectory generation while supporting distributed per-agent trajectory generation for teammate modeling. However, these CTDE-based designs still face scalability issues, as the state and action spaces expand exponentially with the number of agents. Motivated by these challenges, another line of research applies DMs within the DTDE framework. For instance, MA-Diffuser\cite{geng2023diffusion} extends Diffuser\cite{janner2022planning} by learning an independent diffusion planner for each agent, with action-value maximization embedded into the sampling process of a conditional diffusion model. DOM2\cite{li2023beyond} focuses on overcoming the over-conservatism in offline RL by integrating Conservative Q-Learning (CQL)\cite{kumar2020conservative} into Diffusion-QL\cite{wangdiffusion} and introducing a trajectory-based data augmentation strategy to enhance policy diversity and quality. Nevertheless, these independent diffusion models often suffer from poor cooperation among agents. 
To overcome these limitations, our previous work, NetWorld \cite{11707690}, adopts the DTDE framework to enhance scalability and investigates few-shot generalization across heterogeneous wireless networking tasks. Building upon this framework, the present study focuses on diffusion-based planning for MF-TDMA resource allocation and provides a theoretical analysis of the distributional approximation error induced by MF aggregation. This analysis further characterizes the relationship between neighborhood dispersion and diffusion modeling accuracy.
%To overcome these limitations, our work adopts the DTDE framework to enhance scalability and integrates MF communication to facilitate inter-agent coordination, enabling efficient and collaborative resource allocation in complex wireless network environments.

\subsection{Theoretical Analysis of Distributional Error in Diffusion-Based Generation}
As a class of Score-based Generative Models (SGMs), DMs provide an expressive framework for modeling complex distributions, yet their theoretical understanding remains incomplete. For instance, \cite{lee2023convergence} establishes Wasserstein distance guarantees for distributions with bounded support or sufficiently decaying tails under $L^2$-accurate score estimates, and TV bounds under additional smoothness conditions. However, these results rely on strong assumptions such as the Log-Sobolev Inequality (LSI), and obtaining polynomial-time convergence for multi-modal distributions remains an open problem. From a different perspective, \cite{chen2023sampling} employs a Girsanov change-of-measure framework \cite{le2016brownian} and analyzes settings with uniformly Lipschitz-continuous score functions or bounded data support. %Nonetheless, verifying whether the Lipschitz constant scales polynomially with dimension is challenging and highly dependent on the tail behavior of the data distribution. 
Further, \cite{chen2023improved} establishes reverse KL and Wasserstein bounds under refined regularity conditions, employing early stopping to address difficulties near the zero-noise limit. %relaxes these assumptions by requiring smoothness only of $\nabla \log p_{0}$ instead of the entire forward process, thereby sidestepping technical challenges with Novikov’s condition \cite{stummer1993novikov}. This yields reverse KL guarantees and Wasserstein bounds dependent on tail decay. 
Inspired by these advances, our work extends the distributional error analysis to conditional DMs and, for the first time, incorporates MF approximation into the theoretical analysis of diffusion processes. Specifically, we establish convergence guarantees for conditional DMs under MF communication, establishing new distributional error bounds in multi-agent settings.

\section{Preliminaries \& System Model}\label{III}
In this section, we briefly outline the framework of MBRL in multi-agent systems and review the basic principles of DMs. 
%We then present the MF approximation as a scalable approach to simplify agent interactions. Building on these foundations, we show how conditional diffusion model, enhanced by MF-based communication, can be employed as a planner to address 
On this basis, we formulate the sequential decision-making problem of resource allocation in wireless communication networks under the DTDE paradigm, integrating DM-based planning with MF-based communication.

\subsection{Preliminaries}
\subsubsection{Multi-Agent Model-Based Reinforcement Learning}
We consider a partially observable and fully cooperative multi-agent learning problem, where agents with local observations collaborate to complete the task. This setting is formally described as a Decentralized Partially Observable Markov Decision Process (Dec-POMDP) \cite{oliehoek2016concise}: $G = \langle \mathcal{S}, \mathcal{O}, \mathcal{A}, P, \mathcal{R}, N, U, \gamma \rangle$, where $\mathcal{S}$ denotes the global state space, $\mathcal{O} = \mathcal{O}^{(1)} \times \cdots \times \mathcal{O}^{(N)}$ is the joint observation space, $\mathcal{A} = \mathcal{A}^{(1)} \times \cdots \times \mathcal{A}^{(N)}$ is the joint action space, and $\gamma$ is the discount factor. The system contains $N$ agents operating in discrete time steps, starting from an initial global state $s_0 \in \mathcal{S}$ sampled from the distribution $U$. At each timestep $t$, agent $i \in \{1, 2, \cdots, N \}$ receives a local observation $o_t^{(i)} \in \mathcal{O}^{(i)}$ and selects an action $a_t^{(i)} \in \mathcal{A}^{(i)}$. The joint action $\bm{a}_t = (a_t^{(1)}, \cdots, a_t^{(N)}) \in \mathcal{A}$ drives the system to the next state $s_{t+1}$ according to the dynamics function $P(s_{t+1} | s_t, a_t): \mathcal{S} \times \mathcal{A} \rightarrow \mathcal{S}$, and each agent receives a reward $r_{t}^{(i)} \in \mathcal{R}$. In offline RL, agents do not interact with the environment directly. Instead, they rely on a static dataset $\mathcal{D}$, which consists of trajectories $\mathcal{T}$, i.e., sequences of observations, actions and rewards $\mathcal{T} := \{\{o_0^{(i)}\}_{i=1}^{N},\bm{a}_0,\{r_0^{(i)}\}_{i=1}^{N},\{o_1^{(i)}\}_{i=1}^{N},\bm{a}_1,\{r_1^{(i)}\}_{i=1}^{N}\cdots\}$. The objective is to learn the optimal joint policy $\bm{\pi} = (\pi^{(1)},\cdots,\pi^{(N)})$ from $\mathcal{D}$ so as to maximize the expected cumulative discounted sum of rewards as
\begin{equation}
\bm{\pi}^*:= \mathop{\arg\max}\limits_{\bm{\pi}} \mathbb{E}_{\mathcal{T}\sim p_{\bm{\pi}}}\left[\sum\nolimits_{t\ge0}\sum\nolimits_{i=1}^N\gamma^{t}r_t^{(i)}\right]. 
\label{eq:opt_pi}
\end{equation}
In MBRL, agents first estimate the dynamics model from collected environmental transition samples. Based on this learned model, they predict future outcomes and conduct planning to optimize the policy \cite{janner2019trust}. This reduces the need for costly environment interactions and improves both data efficiency and training stability\cite{huang2020model}.

\subsubsection{Diffusion Models}
DMs are a class of deep generative models that aim to approximate the unknown data distribution $\bm{x}_0 \sim p(\bm{x})$ from the dataset $\mathcal{D}$. Instead of directly generating data, DMs define a forward noising process that gradually perturbs data through a predefined It\^{o} SDE \cite{song2021score} as
\begin{equation}
\label{eq:forwadSDE}
    \mathrm{d}\bm{x}_{\tau} = -\frac{\beta_{\tau}}{2} \bm{x}_{\tau} \, \mathrm{d}\tau + \sqrt{\beta_{\tau}} \, \mathrm{d}\bm{w},
\end{equation}
where $\beta_{\tau}: \mathbb{R} \to \mathbb{R} > 0$ is the noise schedule, typically taken to be monotonically increasing of the diffusion timestep $\tau$ \cite{ho2020denoising}, and $\bm{w}$ is the standard Wiener process. For appropriately chosen $\beta_{\tau}$ and sufficiently large diffuse time $T$, it is assumed that the distribution converges to a tractable isotropic Gaussian, i.e. $\bm{x}_T \sim \mathcal{N}(\textbf{0}, \bm{I})$. The learning objective is to recover the original data distribution by reversing this process, which leads to the reverse SDE \cite{anderson1982reverse} as
\begin{equation}
\label{eq:reverseSDE}
    \mathrm{d}\bm{x}_{\tau}\! =\! \left[ -\frac{\beta_{\tau}}{2}\bm{x}_{\tau}\! -\! \beta_{\tau}\nabla_{\bm{x}_{\tau}} \log p_{\tau}(\bm{x}_{\tau}) \right] \mathrm{d}\tau\! +\! \sqrt{\beta_{\tau}} \, \mathrm{d}\bar{\bm{w}},
\end{equation}
where $d\tau$ corresponds to the diffusion step running backward and $d\bar{\bm{w}}$ is the reverse Wiener process. The drift term of Eq. \eqref{eq:reverseSDE} depends on the time-dependent score function $\nabla_{\bm{x}_{\tau}} \log p_{\tau}(\bm{x}_{\tau})$, which is equivalent to the added noise as
\begin{equation}
    \nabla_{\bm{x}_{\tau}} \log p_{\tau}(\bm{x}_{\tau}) = - \frac{\bm{\epsilon}_{\tau}}{\sqrt{1-\bar{\alpha}_{\tau}}},
\end{equation}
where we define $\bar{\alpha}_{\tau} := \exp\left( - \int_{0}^{\tau} \beta_s \, ds \right)$, and $\bm{\epsilon}_{\tau} \sim \mathcal{N}(\textbf{0}, \bm{I})$ is the Gaussian noise injected into the dataset sample $\bm{x}_0$ to produce the perturbed sample $\bm{x}_{\tau}$. Therefore, the reverse denoising process can be learned by optimizing a simplified surrogate objective to predict the added noise $\bm{\epsilon}_{\tau}$ as
\begin{equation}
\label{eq:noise-loss}
\theta^{*}\! =\! \mathop{\arg\min}\limits_{\theta} \mathbb{E}_{\tau \sim \mathcal{U}(0,T),  \bm{x}_{\tau} \sim p_{\tau|0}, \bm{x}_{0} \sim p}\!\left[{\parallel\!\bm{\epsilon}_{\tau}\!-\!\bm{\epsilon}_{\theta}(\bm{x}_{\tau})\!\parallel}^2\right]\!,
\end{equation}
where $\mathcal{U}$ denotes the uniform distribution. With the learned noise function, we can sample data by the reverse denoising SDE in Eq. \eqref{eq:reverseSDE}. 

To guide DMs toward generating data that satisfies specific conditions, conditional information can be incorporated into the diffusion process. While classifier-free guidance \cite{ho2022classifier} provides a unified parameterization of both conditional and unconditional models, its flexibility is limited under complex constraints. In contrast, classifier guidance \cite{dhariwal2021diffusion} has become a widely studied and commonly applied approach\cite{song2021score}. To incorporate conditional information into the generation process, classifier guidance introduces an auxiliary classifier to steer sampling toward desired conditions.
Building upon this principle, the conditional distribution can be derived from the unconditional diffusion prior through Bayes’ rule as
\begin{equation}
    p(\bm{x}_{\tau} \mid \bm{y}) = \frac{p(\bm{x}_{\tau}) \, p(\bm{y} \mid \bm{x}_{\tau})}{p(\bm{y})}. 
\end{equation}
Since $p(\bm{y})$ is independent of $\bm{x}_{\tau}$, Eq.~\eqref{eq:reverseSDE} can be modified as
\begin{align}
\label{eq:reverse-classifier}
    \mathrm{d}\bm{x}_{\tau}\! &=\! \left[\! -\frac{\beta_{\tau}}{2}\bm{x}_{\tau}\! -\! \beta_{\tau}\!\left(\nabla_{\!\bm{x}_{\tau}}\! \log p_{\tau}(\bm{x}_{\tau})\! +\! \nabla_{\!\bm{x}_{\tau}}\! \log p_{\tau}(\bm{y} | \bm{x}_{\tau})\right)\! \right] \mathrm{d}\tau \nonumber\\
    &+ \sqrt{\beta_{\tau}} \, \mathrm{d}\bar{\bm{w}},
\end{align}
% which follows directly from Bayes’ rule:
% \begin{equation}
%     \nabla_{\bm{x}_{\tau}} \!\log p_{\tau}(\bm{x}_{\tau} |\bm{y})\! =\! \nabla_{\bm{x}_{\tau}} \!\log p_{\tau}(\bm{x}_{\tau}) \!+\! \nabla_{\bm{x}_{\tau}} \!\log p_{\tau}(\bm{y} | \bm{x}_{\tau}).
% \end{equation}
Here, the score function $\nabla_{\bm{x}_{\tau}} \log p_{\tau}(\bm{x}_{\tau})$ can be obtained from the diffusion noise model optimized in Eq.~\eqref{eq:noise-loss}, while the conditional likelihood term $\log p_{\tau}(\bm{y}|\bm{x}_{\tau})$ is estimated by an auxiliary network $\mathcal{J}_{\psi}$. The gradients of $\mathcal{J}_{\psi}$ are then incorporated into the sampling process by modifying the drift term of the reverse dynamics as
\begin{align}
\label{eq:reverse model SDE}
    \mathrm{d}\bm{x}_{\tau} \!&=\! \left[\! -\frac{\beta_{\tau}}{2}\bm{x}_{\tau}\!+\! \frac{\beta_{\tau}}{\sqrt{1\!-\!\bar{\alpha}_{\tau}}} \bm{\epsilon}_{\theta}(\bm{x}_{\tau}) \!-\!  \beta_{\tau}\nabla_{\bm{x}_{\tau}} {\mathcal{J}_{\psi}}(\bm{x}_{\tau}) \right] \mathrm{d}\tau \nonumber\\
    \!&+\! \sqrt{\beta_{\tau}} \, \mathrm{d}\bar{\bm{w}}.
\end{align}

\subsection{System Model}
We consider an adaptive, resource-constrained wireless communication network with $N$ nodes, as illustrated in Fig.~\ref{scenario}. Each node is modeled as an autonomous agent that makes sequential decisions on resource allocation. In this setup, direct communication is allowed among 1-hop neighbors via wireless transceivers, while remote interactions rely on multi-hop relaying. Consistent with \cite{9866568}, the capacity of a wireless link is determined by the transmit power, antenna gains, carrier frequency, and inter-node distance.
% The wireless channel follows the Free-Space Path Loss (FSPL) \cite{Harald1946Note} model, where the received power $P_r$ is given by 
% \begin{equation} 
% P_{r} = P_{t} \cdot G_{t} \cdot G_{r} \cdot L_{p}, 
% \label{eq:receivedpower}
% \end{equation}
% with $P_t$ denoting the transmit power and $G_t$, $G_r$ representing antenna gains at the transmitter and receiver, respectively. The path loss $L_{p}$ is defined as
% \begin{equation} 
% L_{p} =\left ( \frac{\lambda }{4\pi d}  \right ) ^{2} , 
% \label{eq:pathloss}
% \end{equation}
% where $d$ is the transmission distance and $\lambda ={c}/{f_{c}}$ is the wavelength, with $f_{c}$ representing the carrier frequency and $c$ denoting the speed of light. The FSPL model represents an ideal Line-Of-Sight (LOS) propagation environment, providing stable and predictable signal attenuation for accurate performance evaluation.
Every node maintains a packet queue for storing generated data, and unacknowledged packets are appended to the end of the cache queue of forwarding nodes. Meanwhile, the network exhibits heterogeneous traffic, where some nodes produce high-rate flows while others generate packets at lower rates, and the imbalance patterns are not known beforehand. To enable efficient spectrum reuse and mitigate inter-node interference, we adopt the MF-TDMA \cite{9866568} protocol, in which the available resources are partitioned into $M$ orthogonal time slots and $L$ frequency channels, yielding $M \times L$ orthogonal Resource Blocks (RBs) per frame. Each transmission frame includes two stages (i.e., network management and traffic transmission). In the first stage, each agent independently determines its RB demand based on local observations and $1$-hop information. The network management mechanism then collects and normalizes these demands to satisfy the total RB capacity of $M \times L$, followed by random RB assignment. Notably, this normalization is part of the underlying network management procedure rather than the agents' decentralized planning process. In the second stage, packets are transmitted within the allocated RBs, effectively reducing collision risk and improving spectral efficiency.

\begin{figure}[tbp]
\centerline{\includegraphics[width=0.5\textwidth]{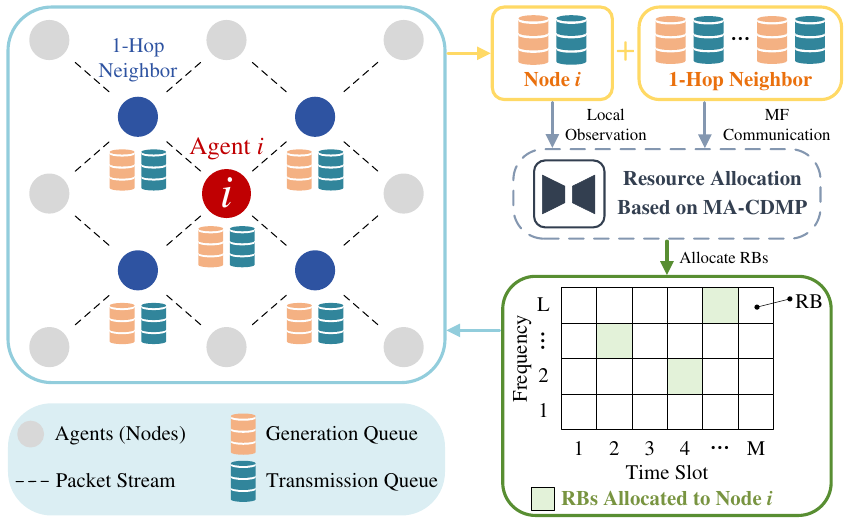}}
\caption{Overview of the multi-node wireless topology with MF-TDMA-based MAC scheduling.}
\label{scenario}
\vspace{-0.2cm}
\end{figure}

We formulate the resource allocation task as a Dec-POMDP\cite{zhu2024madiff}, and define the decision interval as one frame. Specifically, at each timestep (i.e., one frame) $t$, agent $i$ obtains a local observation $o_{t}^{(i)} := \left[ {\rm gen}_{t}^{(i)}, {\rm gen}_{\max,t}^{(i)}, {\rm tran}_{t}^{(i)}, {\rm tran}_{\max,t}^{(i)} \right]$, where ${\rm gen}_{t}^{(i)}$ and ${\rm tran}_{t}^{(i)}$ denote the real-time lengths of the generation and transmission queues, respectively, while ${\rm gen}_{\max,t}^{(i)}$ and ${\rm tran}_{\max,t}^{(i)}$ represent their corresponding maximum values within frame $t$. To efficiently enhance coordination under limited communication resources, each agent only exchanges its local observation vector with its 1-hop neighbors, rather than collecting global states or joint observations from all agents. The received neighbor observations are then summarized by the MF mechanism, which provides a lightweight local descriptor of nearby traffic and queue dynamics as
\begin{equation}
\label{eq:MF-observation} \overline{{o}}^{(i)}_t \!:=\! \frac{1}{|\mathcal{N}_{i}|} \sum_{j \in \mathcal{N}_i} {o}^{(j)}_t,
\end{equation}
where $\mathcal{N}_{i}$ denotes the 1-hop neighbors of agent $i$. To support long-term decision-making, the action $a_t^{(i)} \in \mathbb{R}$ is generated based on both the local observation sequence and the MF communication sequence, determining the number of RBs requested by the agent. The reward of agent $i$ is defined as 
\begin{equation} 
r_{t}^{(i)} :=- d_{t}^{(i)} , 
\label{eq:reward}
\end{equation}
where $d_{t}^{(i)}$ indicates the average time delay for packets received by node $i$. By minimizing the average delay, we can simultaneously improve throughput and reduce packet loss, thereby enhancing the overall Quality of Service (QoS) without requiring complex weight tuning in multi-objective optimization \cite{meng_conditional_2025a}. In this paper, since nodes exhibit similar structural and operational features, we model them as homogeneous agents. 

For each agent $i$, a policy $\pi^{(i)}$ is learned from its observation sequence and MF communication sequence to collectively maximize the long-term network utility as formulated in Eq.~\eqref{eq:opt_pi}. Given the large-scale characteristics of the resource allocation problem in such wireless networks, we devise a decentralized MA-CDMP approach with scalable coordination and low communication overhead, laying the foundation for the algorithmic design presented in the following section.

\section{Resource Allocation With MA-CDMP}\label{IV}
In this section, we discuss how to apply conditional DMs and MF communication to guide the generation of trajectories with high confidence toward superior cumulative resource allocation rewards. The entire procedure for training and implementation of the proposed MA-CDMP framework is illustrated in Fig.~\ref{MACDMP}.  

\begin{figure*}[tbp]
\setlength{\abovecaptionskip}{0cm} %调整图片标题与图距离
\setlength{\belowcaptionskip}{-0.5cm} %调整图片标题与下文距离
\centering
\includegraphics[width=0.9\textwidth]{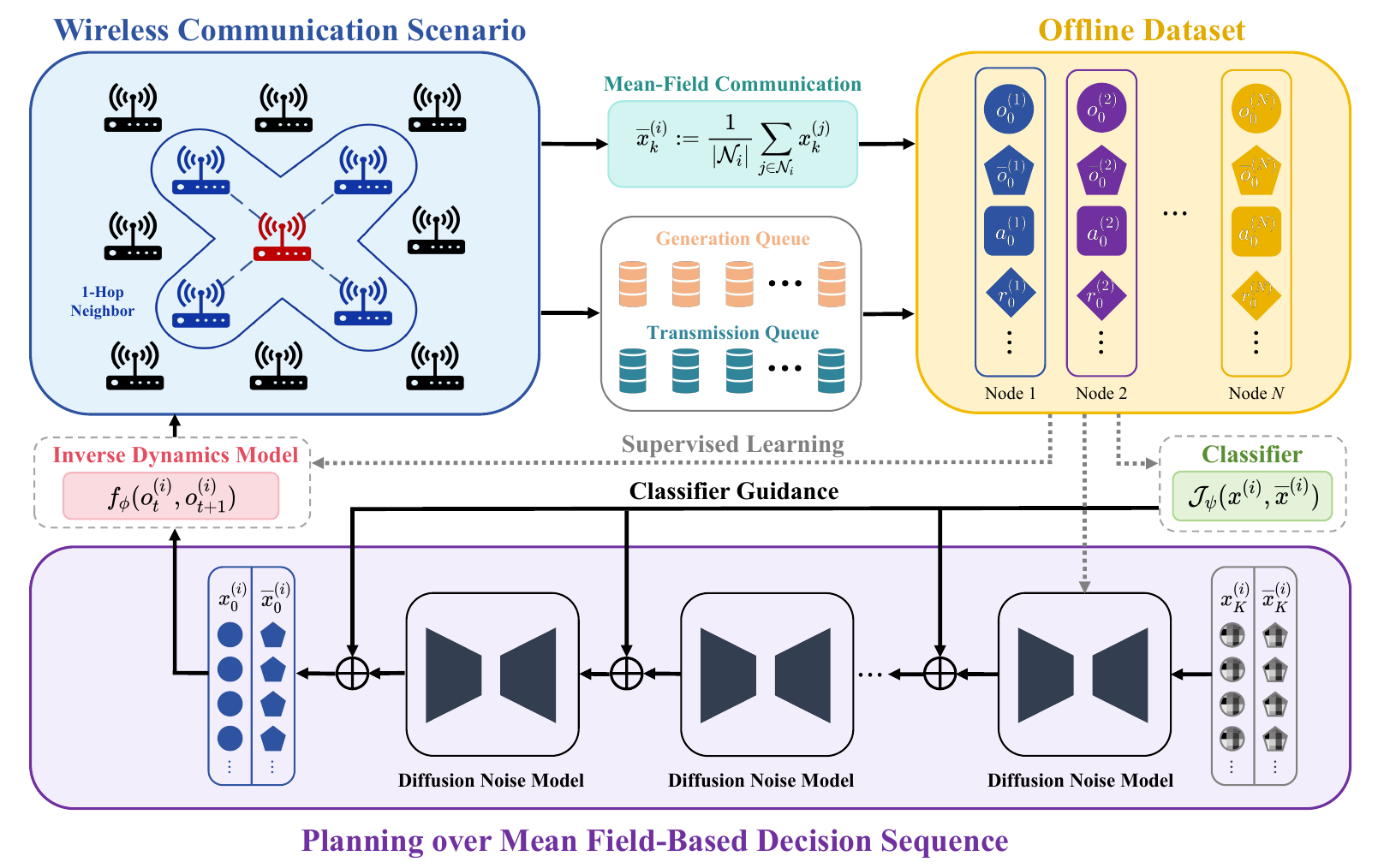}
\caption{The illustration of the MA-CDMP algorithm for resource allocation in wireless communication networks. }
\label{MACDMP}
\vspace{-0.2cm}
\end{figure*}

\subsection{Overview of MA-CDMP}
The proposed MA-CDMP encompasses a DM-based dynamics modeling stage, followed by a trajectory-driven planning optimization stage. In the first modeling stage, considering that the discrete and highly fluctuating nature of RB allocation actions increases the difficulty of prediction and modeling, we define $\bm{x}^{(i)}_{k}$ as an observation sequence of length $H$ with noise level $k$, namely
\begin{equation}
\bm{x}^{(i)}_{k}:={(o_t^{(i)}, o_{t+1}^{(i)}, \cdots, o_{t+H-1}^{(i)})}_{k}.
\end{equation}
Meanwhile, to guide the denoising process toward generating trajectories with high cumulative rewards, the condition $\bm{y}^{(i)}$ is defined as the return of the trajectory ${\rm{Return}}^{(i)}$, that is
\begin{equation}
\label{eq:return}
\bm{y}^{(i)}:= {\rm{Return}}^{(i)}:=\sum\nolimits_{t'=t}^{t+H-1}{\gamma}^{t'-t}r_{t'}^{(i)}.
\end{equation}
%In the planning phase of MBRL, it is often necessary to simulate or explore multiple sequences of actions and observations to evaluate long-term outcomes and refine decision-making over extended horizons. 
To this end, we formulate trajectory modeling in MARL as a conditional generative problem based on DMs with classifier guidance \cite{janner2022planning}, namely
\begin{equation}
\max\limits_{\theta, \psi}\mathbb{E}_{\mathcal{T} \sim \mathcal{D}}\left[\log p_{\theta, \psi}(\bm{x}^{(i)}_0|\bm{y}^{(i)})\right].
\label{equa}
\end{equation}
In this formulation, after characterizing the trajectory distribution $p_{\theta,\psi}$ with the diffusion noise model $\bm{\epsilon}_{\theta}$ and the classifier $\mathcal{J}_{\psi}$, we can generate portions of agent $i$'s trajectory $\bm{x}^{(i)}_0$ conditioned on information $\bm{y}^{(i)}$. For practical implementation, we discretize the continuous forward diffusion SDE in Eq.~\eqref{eq:forwadSDE} and the reverse denoising SDE in Eq.~\eqref{eq:reverse-classifier} into $K$ steps, where each diffusion step is indexed by $k \in \{ 1, \cdots, K\}$. Following \cite{ho2020denoising}, we denote $\bm{x}_{k} := \bm{x}_{{kT}/{K}}$, $\beta_{k} := \beta_{{kT}/{K}}$, and further define $\alpha_{k} := 1 - \beta_{k}$, $\bar{\alpha}_{k} := \prod_{s=1}^{k} \alpha_{s}$. 

When evaluating trajectories, the classifier operates on the joint observation sequence $\widetilde{\bm{x}}_{k}=[\bm{x}_k^{(1)},\cdots,\bm{x}_k^{(N)}]$, as all agents simultaneously make strategic decisions and assess their value functions over the joint trajectory. However, since the number of possible joint trajectory configurations grows exponentially with the number of agents, directly modeling the classifier $\mathcal{J}_{\psi}(\widetilde{\bm{x}}_{k})$ becomes computationally prohibitive. To overcome this challenge, we decompose the classifier into components that rely solely on pairwise local interactions \cite{yang2018mean} with the agent's 1-hop neighbors as
\begin{equation}
    \textstyle
    \mathcal{J}_{\psi}(\widetilde{\bm{x}}_{k}^{(i)}) := \frac{1}{|\mathcal{N}_{i}|} \sum_{j \in \mathcal{N}_i} \mathcal{J}_{\psi}(\bm{x}^{(i)}_{k},\bm{x}^{(j)}_{k}).
\end{equation}
This factorization substantially reduces complexity while still preserving implicit global interactions. The pairwise interaction can be further approximated via MF theory \cite{stanley1971phase}, by replacing neighbors with their mean observation as
\begin{equation}
\label{eq:MF}
\mathcal{J}_{\psi}(\widetilde{\bm{x}}_{k}^{(i)}) \!\sim\! \mathcal{J}_{\psi}({\bm{x}}_{k}^{(i)}, \overline{\bm{x}}_{k}^{(i)}),  \quad\overline{\bm{x}}_{k}^{(i)} \!:=\! \frac{1}{|\mathcal{N}_{i}|} \sum_{j \in \mathcal{N}_i} \bm{x}^{(j)}_{k}.
\end{equation}
We show that under mild conditions, the approximation error is bounded within a symmetric interval, with the further derivation provided in Section~\ref{sec:drift error}.

Subsequently, given a learned inverse dynamics model $f_{\phi}(o_t^{(i)}, o_{t+1}^{(i)})$ \cite{allen2021learning}, the policy at any timestep $t$ in $\bm{x}^{(i)}_{0}$ can be recovered by estimating the action that leads to the observation transition, namely
\begin{equation}
a_t^{(i)}:=f_{\phi}(o_t^{(i)}, o_{t+1}^{(i)}).
\end{equation}
We leave the details of the training procedure in Section \ref{sec:training}, while the implementation process of the frameworks will be enumerated in Section \ref{sec:implementation}.
%This trajectory-driven framework thus enables planning-based policy optimization in multi-agent systems.

\subsection{Training of Core Components in MA-CDMP}\label{sec:training}
Building upon the modeling framework introduced in the previous subsection, we describe the training of the key components that enable effective trajectory-driven planning. Since the agents are modeled as homogeneous, we adopt parameter sharing across agents to enhance efficiency. The classifier $\mathcal{J}_{\psi}({\bm{x}}_{k}^{(i)}, \overline{\bm{x}}_{k}^{(i)})$ is trained in a supervised manner to estimate the cumulative return of a trajectory. To maintain consistency, the diffusion noise model $\bm{\epsilon}_{\theta}({\bm{x}}_{k}^{(i)}, \overline{\bm{x}}_{k}^{(i)})$ takes both the local observation sequence  $\bm{x}_{k}^{(i)}$ and the MF observation sequence $\overline{\bm{x}}_{k}^{(i)}$ as inputs, thereby capturing their joint distribution. In parallel, the inverse dynamics model $f_{\phi}(o_t^{(i)}, o_{t+1}^{(i)})$ is trained under a supervised objective, where the true action $a_t^{(i)}$ serves as the ground-truth signal. Given an offline dataset $\mathcal{D}$ containing observation–action trajectories labeled with rewards, the overall training is conducted by jointly optimizing the classifier loss, reverse diffusion loss, and inverse dynamics loss:
\begin{align}
\label{eq:MA-CDMP}
&\mathcal{L}_{\rm{MA-CDMP}}(\psi, \theta,\phi):= \mathbb{E}_{\mathcal{T}\in \mathcal{D}}\left[{\parallel a_t^{(i)}\!-\!f_{\phi}(o_t^{(i)}, o_{t+1}^{(i)})\parallel}^2\right] \nonumber \\ 
&\quad +\mathbb{E}_{k, \mathcal{T}\in \mathcal{D}}\left[{\parallel \bm{\epsilon}_k-\bm{\epsilon}_{\theta}({\bm{x}}_{k}^{(i)}, \overline{\bm{x}}_{k}^{(i)})\parallel}^2\right]\nonumber \\ 
&\quad + \mathbb{E}_{k, \mathcal{T}\in \mathcal{D}}\left[{\parallel {\bm{y}}^{(i)}-\mathcal{J}_{\psi}({\bm{x}}_{k}^{(i)}, \overline{\bm{x}}_{k}^{(i)})\parallel}^2\right]
\end{align}
The overall offline training procedure of MA-CDMP is summarized in Algorithm~\ref{training}.

\begin{algorithm}[!t]
    \caption{Offline training of MA-CDMP.}
    \label{training}
    \renewcommand{\algorithmicrequire}{\textbf{Input:}}
    \renewcommand{\algorithmicensure}{Initialize}
    \renewcommand{\algorithmiccomment}[1]{\hfill $\triangleright$ #1}
    
    \begin{algorithmic}[1]
        \REQUIRE Offline trajectory dataset $\mathcal{D}$, planning horizon $H$, diffusion steps $K$
        \ENSURE Classifier $\mathcal{J}_{\psi}$, diffusion noise model $\bm{\epsilon}_{\theta}$ and inverse dynamics model $f_{\phi}$ with random parameters $\psi$, $\theta$, $\phi$
        
        \FOR{each epoch}
          \FOR{each step}
            \STATE Sample a mini-batch $\mathcal{B}$ of $\bm{x}_0^{(i)}$ in $\mathcal{D}$
            \STATE Predict actions $f_{\phi}(o_t^{(i)}, o_{t+1}^{(i)})$ for each observation transition pair in $\mathcal{B}$
            \STATE Compute the MF observation $\overline{\bm{x}}_{0}^{(i)}$ by Eq.~\eqref{eq:MF}
            \STATE $\bm{y}^{(i)} \gets {\rm{Return}}^{(i)}$ \\ \COMMENT{Obtain the condition information by Eq. \eqref{eq:return}}  
            \STATE $k \sim{\rm{Uniform}}(\{1,2,\cdots,K\})$ 
            \STATE $\bm{\epsilon}_k \sim \mathcal{N}(\textbf{0}, \bm{I})$
            \STATE $ \left[ \bm{x}_{k}^{(i)}, \overline{\bm{x}}_{k}^{(i)}\right] \gets\sqrt{\bar{\alpha}_k}\left[ \bm{x}_0^{(i)}, \overline{\bm{x}}_0^{(i)} \right]+\sqrt{1-\bar{\alpha}_k}\bm{\epsilon}_k$ \\ \COMMENT{Apply the forward diffusion process}
            \STATE Predict noise $\bm{\epsilon}_{\theta}({\bm{x}}_{k}^{(i)}, \overline{\bm{x}}_{k}^{(i)})$
            \STATE Predict return $\mathcal{J}_{\psi}({\bm{x}}_{k}^{(i)}, \overline{\bm{x}}_{k}^{(i)})$
            \STATE Update the classifier $\mathcal{J}_{\psi}$, diffusion noise model $\epsilon_{\theta}$ and inverse dynamics model $f_{\phi}$ by Eq.~\eqref{eq:MA-CDMP}
            \ENDFOR
        \ENDFOR
    
    \end{algorithmic}
\end{algorithm}

\subsection{Implementation with MF-Enhanced Classifier Guidance}\label{sec:implementation}
We implement the well-trained classifier, diffusion noise model, and inverse dynamics model to generate trajectories for policy optimization. Formally, the trajectory generation process begins from Gaussian noise $\left[ \bm{x}_{K}^{(i)}, \overline{\bm{x}}_{K}^{(i)}\right] \sim \mathcal{N}(\textbf{0}, \bm{I})$ and iteratively refines $\left[ \bm{x}_{k}^{(i)}, \overline{\bm{x}}_{k}^{(i)}\right]$ into $\left[ \bm{x}_{k-1}^{(i)}, \overline{\bm{x}}_{k-1}^{(i)}\right]$ at each diffusion step $k$, with the mean given by
\begin{align}
\label{eq:reverse mean}
    \bm{\mu}_{\theta, \psi} &=\frac{1}{\sqrt{\alpha_k}}\left[\left[ \bm{x}_{k}^{(i)}, \overline{\bm{x}}_{k}^{(i)}\right]\! -\! \frac{1\!-\!\alpha_k}{\sqrt{1\!-\!\bar{\alpha}_k}}\bm{\epsilon}_{\theta}(\bm{x}_{k}^{(i)}, \overline{\bm{x}}_{k}^{(i)})\right] \nonumber\\
    &+ \zeta \cdot \bm{\Sigma}_k \cdot \nabla {\mathcal{J}_{\psi}}( \bm{x}_{k}^{(i)}, \overline{\bm{x}}_{k}^{(i)}),  
\end{align}
where $\bm{\Sigma}_k$ denotes the stepwise reverse-diffusion variance prescribed by the sampling schedule (i.e., $\frac{1 - \bar{\alpha}_{k-1}}{1 - \bar{\alpha}_{k}} \, \beta_{k}$ in DDPM-style discretizations \cite{ho2020denoising}), and $\zeta$ is a coefficient that balances the influence of the classifier guidance, controlling the strength of conditional enforcement during sampling. In addition, to ensure consistency with the historical data, the first observation of the trajectory is set according to the current state throughout all diffusion steps. After completing $K$ iterations of the denoising process, we can infer the next predicted local observation and determine the optimal action using our inverse dynamics model $f_{\phi}$. This procedure repeats in a standard receding-horizon control loop described in Algorithm~\ref{implementing}.

\begin{algorithm}[!t]
    \caption{Implementation of MA-CDMP.}
    \label{implementing}
    \renewcommand{\algorithmicrequire}{\textbf{Input:}}
    \renewcommand{\algorithmicensure}{Initialize}
    \renewcommand{\algorithmiccomment}[1]{\hfill $\triangleright$ #1}
    
    \begin{algorithmic}[1]
        \REQUIRE Classifier $\mathcal{J}_{\psi}$, diffusion noise model $\bm{\epsilon}_{\theta}$, inverse dynamics model $f_{\phi}$, planning horizon $H$, diffusion steps $K$, conditional guidance scale $\zeta$, number of nodes $N$, total available RB capacity $ML$
        \ENSURE $t\gets0$
        
        \WHILE{not done}
            \FOR{$i = 1 \cdots N$}
                \STATE Observe $o_t^{(i)}$; Get MF communication $\overline{o}_t^{(i)}$ from 1-hop neighbors $\mathcal{N}_i$
                \STATE Initialize $\left[ \bm{x}_{K}^{(i)}, \overline{\bm{x}}_{K}^{(i)}\right] \sim \mathcal{N}(\textbf{0}, \bm{I})$
                \STATE $\left[ \bm{x}_{K}^{(i)}[0], \overline{\bm{x}}_{K}^{(i)}[0]\right]\gets \left[ o_t^{(i)}, \overline{o}_t^{(i)} \right]$ \\ \COMMENT{Constrain consistency}
                \FOR{$k=K\cdots1$}
                    \STATE $(\bm{\mu}_{k-1}, \bm{\Sigma}_{k-1})\gets{\rm{Denoise}}(\left[ \bm{x}_{k}^{(i)}, \overline{\bm{x}}_{k}^{(i)}\right],\bm{\epsilon}_{\theta}, \mathcal{J}_{\psi})$  \\ \COMMENT{Apply the denoising process by Eq.~\eqref{eq:reverse mean}}
                    \STATE $\left[ \bm{x}_{k-1}^{(i)}, \overline{\bm{x}}_{k-1}^{(i)}\right]\sim \mathcal{N}(\bm{\mu}_{k-1}, \bm{\Sigma}_{k-1})$
                    \STATE $\left[ \bm{x}_{k-1}^{(i)}[0], \overline{\bm{x}}_{k-1}^{(i)}[0]\right]\gets \left[ o_t^{(i)}, \overline{o}_t^{(i)} \right]$ \\ \COMMENT{Constrain consistency}
                \ENDFOR
                \STATE Extract $(o_t^{(i)}, o_{t+1}^{(i)})$ from $\bm{x}_0^{(i)}$
                \STATE Predict $a_t^{(i)}=f_{\phi}(o_t^{(i)}, o_{t+1}^{(i)})$
            \ENDFOR
            \STATE Obtain the feasible RB allocations $\{\hat{a}_t^{(i)}\}_{i=1}^{N}$ by normalizing demands and randomly assigning RBs through the network management procedure
            % \STATE $\hat{a}_t^{(i)} \gets \frac{ML}{\sum_{i=1}^{N} a_t^{(i)}}\, a_t^{(i)}\ \ \forall i$ \\ \COMMENT{Normalize so $\sum_i \hat{a}_t^{(i)}=ML$}
            \STATE Execute $\left\{ \hat{a}_t^{(i)} \mid i=1,\cdots,N \right\}$ ; $t\gets t+1$
        \ENDWHILE
    
    \end{algorithmic}
\end{algorithm}

\section{Error Guarantee for MA-CDMP}\label{V}
In this section, we focus on the theoretical guarantee for the convergence of the proposed MA-CDMP framework. Specifically, we first analyze the estimation error for the drift term of the diffusion SDE under MF approximation. Building on this, we provide an upper bound on the distributional error of conditional diffusion-based generation.

\subsection{Theoretical Upper Bound on the Drift Term Approximation Error}\label{sec:drift error}
Consistent with existing literature \cite{yang2018mean}, we first make the following assumptions.
\begin{assumption}
\label{assump:smooth}
 The classifier $\mathcal{J}_{\psi}$ and diffusion noise model $\bm{\epsilon}_{\theta}$ are L-smooth, with gradients that are Lipschitz-continuous with constants $L_{\mathcal{J}}$ and $L_{\bm{\epsilon}}$, respectively:
\begin{subequations}
\begin{equation}
\label{eq:classifier continuity}
\parallel \!\nabla\! {\mathcal{J}}(\bm{x}^{(i)}\!,\! \bm{x}^{(j)}) \!-\! \nabla\! {\mathcal{J}}(\bm{x}^{(i)}\!,\! \overline{\bm{x}}^{(i)})\!\parallel
\leq L_{{\mathcal{J}}} \parallel\! \bm{x}^{(j)} \!-\! \overline{\bm{x}}^{(i)}\!\parallel,
\end{equation}
\begin{equation}
\parallel\nabla \bm{\epsilon}(\bm{x}^{(i)}\!,\! \bm{x}^{(j)}) - \nabla \bm{\epsilon}(\bm{x}^{(i)}\!,\! \overline{\bm{x}}^{(i)})\!\parallel
\leq L_{\bm{\epsilon}}\parallel\! \bm{x}^{(j)} \!-\! \overline{\bm{x}}^{(i)}\!\parallel.
\end{equation}
\end{subequations}
\end{assumption}
\begin{assumption}
\label{assump:difference}

The local neighborhood dispersion around the MF trajectory is bounded as
\begin{equation}
\|\bm{x}^{(j)} - \overline{\bm{x}}^{(i)}\|^2 = \|\bm{x}^{(j)}\|^2 + \|\overline{\bm{x}}^{(i)}\|^2 - 2(\bm{x}^{(j)})^\top \overline{\bm{x}}^{(i)} \leq C,
\end{equation}
where $C$ is a constant bound related to the magnitude of observation vectors.     
\end{assumption}

These assumptions are local regularity conditions over the finite-horizon normalized trajectory space. In the considered MF-TDMA resource-allocation task, the observation variables are queue-related and RB-related quantities, which are naturally bounded by finite buffers, finite RB capacity, and the fixed planning horizon.

\begin{lemma}
\label{lemma:drift}
Under Assumptions \ref{assump:smooth} and \ref{assump:difference}, for the MF communication mechanism, the drift term approximation error of the reverse-time SDE in Eq.~\eqref{eq:reverse model SDE} admits an upper bound as
    \begin{equation}
        \delta_{\text{drift}} \le \frac{C L_{\bm{\epsilon}} \beta_{\tau}}{\sqrt{1 - \bar{\alpha}_{\tau}}} + \sqrt{C} L_{\mathcal{J}} \beta_{\tau}.
\end{equation}
\end{lemma}
\begin{proof}
Following the Taylor-remainder analysis in MF communication \cite{hao2023gat}, the smooth pairwise interaction function $\epsilon_\theta(\bm{x}^{(i)},\overline{\bm{x}}^{(i)})$ could be expanded around $\overline{\bm{x}}^{(i)}$, and the first-order term vanishes after averaging over neighbors. Therefore, the remaining approximation error is controlled by the second-order remainder. Under Assumptions \ref{assump:smooth} and \ref{assump:difference}, this yields
% \begin{equation}
% \parallel{\mathcal{J}}(\bm{x}^{(i)}, \overline{\bm{x}}^{(i)}) - \mathcal{J}(\widetilde{\bm{x}}^{(i)})\parallel \leq C L_{{\mathcal{J}}},
% \end{equation}
\begin{equation}
\label{eq:epsilon error}
\parallel\bm{\epsilon}(\bm{x}^{(i)}, \overline{\bm{x}}^{(i)}) - \bm{\epsilon}(\widetilde{\bm{x}}^{(i)})\parallel \leq C L_{{\bm{\epsilon}}}.
\end{equation}
Next, considering the gradient of the classifier over the observation trajectory, we have
\begin{align}
&\parallel\nabla\! \mathcal{J}(\widetilde{\bm{x}}^{(i)})  
\!-\!  \nabla\! \mathcal{J}(\bm{x}^{(i)}\!,\! \overline{\bm{x}}^{(i)}) \parallel \nonumber \\
=& \left\| \frac{1}{|\mathcal{N}_{i}|} \sum_{j \in \mathcal{N}_{i}} 
\big( \nabla\! {\mathcal{J}}(\bm{x}^{(i)}\!,\! \bm{x}^{(j)}) 
\!-\! \nabla \!{\mathcal{J}}(\bm{x}^{(i)}\!,\! \overline{\bm{x}}^{(i)}) \big) \right\| \nonumber \\
\leq& \frac{1}{|\mathcal{N}_{i}|} \sum_{j \in \mathcal{N}_{i}} 
\parallel \nabla \!{\mathcal{J}}(\bm{x}^{(i)}\!,\! \bm{x}^{(j)}) 
\!-\! \nabla\! {\mathcal{J}}(\bm{x}^{(i)}\!,\! \overline{\bm{x}}^{(i)}) \parallel \nonumber \\
\leq& \frac{1}{|\mathcal{N}_{i}|} \sum_{j \in \mathcal{N}_{i}} 
L_{{\mathcal{J}}} \cdot \parallel \bm{x}^{(j)} - \overline{\bm{x}}^{(i)} \parallel,
\end{align}
where the third line follows from the triangle inequality, while the fourth line applies the Lipschitz-continuity condition in Eq.~\eqref{eq:classifier continuity}. With the bounded observation trajectory difference, this yields
\begin{equation}
\label{eq:classifier error}
\parallel\nabla\! \mathcal{J}(\widetilde{\bm{x}}^{(i)}) 
\!-\!  \nabla\! \mathcal{J}(\bm{x}^{(i)},\overline{\bm{x}}^{(i)}) \parallel\, \leq L_{\mathcal{J}}\sqrt{C}.
\end{equation}
Finally, combining Eq.~\eqref{eq:epsilon error} and Eq.~\eqref{eq:classifier error}, we have the lemma.
%the drift term approximation error of Eq.~\eqref{eq:reverse model SDE}, denoted as $\delta_{\text{drift}}$, can be bounded in the following lemma.

\end{proof}

\subsection{Theoretical Upper Bound on the SDE-Based Generative Distribution Approximation Error} 

Following the early-stopping strategy in \cite{chen2023improved}, we introduce a strictly positive terminal diffusion time $\tau_0>0$ to avoid the singularity near $\tau=0$. Accordingly, the theoretical analysis is restricted to the interval $[\tau_0,T]$, where $1-\bar{\alpha}_\tau>0$, ensuring the finiteness of the drift-error integral under the stated regularity assumptions. The resulting bound therefore characterizes the distributional approximation error of the truncated reverse diffusion process.

\begin{theorem}
    \label{thm: kl_bound}
    Under the MF communication mechanism, the conditional diffusion modeling error between the forward diffusion marginal $p_{\tau_0}$ and the generated distribution $q_{\tau_0}$ of the truncated reverse diffusion process is upper bounded as
    \begin{equation}
        D_{\text{KL}}(p_{{\tau_0}}\|{q}_{{\tau_0}}) \le \frac{1}{2} M_{2} e^{-\bar{\beta}_T} +\frac{1}{2}\int_{{\tau_0}}^{T} \mathbb{E}\left[\tfrac{\delta_{\text{drift}}^{2}}{\beta_{\tau}}\right] \mathrm{d}\tau,
    \end{equation}
    where we denote $M_2:=\mathbb{E}_{\bm{x}_{\tau_0}\sim p_{\tau_0}}\|\bm{x}_{\tau_0}\|^2<\infty$ as the second-order moment at the early-stopping time, and $\bar{\beta}_T := \int_{\tau_0}^T \beta_{\tau} \mathrm{d}{\tau}$.
\end{theorem}

Before proving Theorem~\ref{thm: kl_bound}, we introduce several key lemmas that establish the necessary theoretical groundwork.

\begin{lemma}\label{lemma:OU process}
    Consider the Ornstein–Uhlenbeck (OU) process with time-varying coefficient in Eq.~\eqref{eq:forwadSDE}. The KL divergence between $p_T$ and the standard Gaussian distribution $\rho$ satisfies
    \begin{equation}
        D_{\text{KL}}(p_{T}\,\|\,\rho) \;\lesssim\; \tfrac{1}{2} M_{2} \cdot e^{-\bar{\beta}_T}.
    \end{equation}
\end{lemma}
\noindent The proof of Lemma \ref{lemma:OU process} is provided in Appendix~\ref{proof:lemma 2}. Notably, this estimate does not depend on the initial distance $D_{\text{KL}}(p_{\tau_0}\|\rho)$, as in \cite{chen2023sampling}, and further extends the derivation in \cite{chen2023improved} to the time-varying OU process.

We define the relative Fisher information between $p_{\tau}$ and $q_{\tau}$ by $J(p_{\tau} \,\|\, q_{\tau}) = \int p_{\tau}(X_{\tau}) \, \left\| \nabla \log \frac{p_{\tau}(X_{\tau})}{q_{\tau}(X_{\tau})} \right\|^2 \, \mathrm{d}X_{\tau}$, and will have the following lemma.
\begin{lemma}\label{lemma:relative Fisher information}
    Consider the following two It\^{o} processes with the same diffusion term $g(\tau)$
    \begin{align}
        \mathrm{d}X_{\tau} &= F_{1}(X_{\tau})\,\mathrm{d}{\tau} + g({\tau})\,\mathrm{d}\bm{w}, \\
        \mathrm{d}Y_{\tau} &= F_{2}(Y_{\tau})\,\mathrm{d}{\tau} + g({\tau})\,\mathrm{d}\bm{w},
    \end{align}
    where $F_{1}(X_{\tau}), F_{2}(Y_{\tau}), g({\tau})$ are continuous functions. We assume the uniqueness and regularity condition:
    \begin{itemize}
        \item The two SDEs have unique solutions.
        \item $X_{\tau}, Y_{\tau}$ admit densities $p_{\tau}, q_{\tau} \in C^2(\mathbb{R}^d)$ for $T > \tau > 0$.
    \end{itemize}
    Then the evolution of $D_{\text{KL}}(p_{\tau} \,\|\, q_{\tau})$ is given by
    \begin{align}\label{eq:lemma 2}
        \frac{\partial}{\partial {\tau}} D_{\text{KL}}(&p_{\tau} \,\|\, q_{\tau})  = -\tfrac{1}{2}g({\tau})^2 J(p_{\tau} \,\|\, q_{\tau}) \nonumber \\
        &+ \mathbb{E} \!\left[ \left\langle F_{1}(X_{\tau}) - F_{2}(X_{\tau}), \nabla \log \frac{p_{\tau}(X_{\tau})}{q_{\tau}(X_{\tau})} \right\rangle \right].
    \end{align}
\end{lemma}
\noindent The proof of Lemma \ref{lemma:relative Fisher information} is included in Appendix~\ref{proof:lemma 3}. 

Now, we are ready to prove Theorem \ref{thm: kl_bound}.
\begin{proof}
    By applying the chain rule of KL divergence, the error can be decomposed as 
\begin{align}
\label{eq:KL chain rule}
    D_{\text{KL}}(p_{{\tau_0}}\|&{q}_{{\tau_0}}) \\
    \leq & D_{\text{KL}}(p_{T}\|{q}_{T}) + \mathbb{E}_{\bm{x}_T} D_{\text{KL}}\!\left(p(\bm{x}_{\tau_0}|\bm{x}_T)\|q(\bm{x}_{\tau_0}|\bm{x}_T)\right).\nonumber
\end{align}
Since the reverse denoising process begins from Gaussian noise, i.e., $q_T \sim \mathcal{N}(\textbf{0}, \bm{I})$, the first term can be bounded by Lemma~\ref{lemma:OU process}. For the second term in Eq.~\eqref{eq:KL chain rule}, by applying Lemma~\ref{lemma:relative Fisher information}
and the weighted Young inequality, we yield
\begin{align}
    \frac{\partial}{\partial {\tau}} &D_{\text{KL}}\!\left(p_{\tau}\|q_{\tau}\right)\nonumber\\ 
    &\leq -\tfrac{\beta_{\tau}}{2} J\!\left(p_{\tau}\|q_{\tau}\right) 
    \!+\! \left( \tfrac{1}{2\beta_{\tau}} \mathbb{E}\!\left\| F_{1} \!-\! F_{2} \right\|^{2} 
    \!+\! \tfrac{\beta_{\tau}}{2} J\!\left(p_{\tau}\|q_{\tau}\right) \right) \nonumber \\
    &= \tfrac{1}{2\beta_{\tau}} \mathbb{E}\!\left\| F_{1} - F_{2} \right\|^{2}.
\end{align}
Integrating both sides of this inequality with respect to $\tau$ over the interval $[{\tau_0},T]$ further yields
\begin{equation}
    \mathbb{E}_{\bm{x}_T} D_{\text{KL}}\!\left(p(\bm{x}_{\tau_0}|\bm{x}_T)\|q(\bm{x}_{\tau_0}|\bm{x}_T)\right) 
\!\leq\! \textstyle\int_{{\tau_0}}^{T}\! 
\mathbb{E}
\tfrac{1}{2\beta_{\tau}}\left\| \big( F_{1} \!-\! F_{2}\big) \right\|^{2} \!\mathrm{d}\tau.
\end{equation}
By substituting Eq.~\eqref{eq:KL chain rule} and incorporating the upper bound of the drift term error in Lemma \ref{lemma:drift}, we obtain
the theorem.
\end{proof}

\begin{remark}
Theorem~\ref{thm: kl_bound} shows that, with an appropriate noise schedule and well-trained model parameters, the conditional diffusion modeling error of MA-CDMP under MF communication can be explicitly bounded. Our proof employs a Fokker-Planck-based differential inequality~\cite{risken1989fokker} with relative Fisher information~\cite{sun2017relative} to track the KL-divergence evolution along the reverse SDE. The bound directly connects the final distributional error with the MF-induced drift approximation error, thereby characterizing how neighborhood dispersion affects diffusion-based multi-agent dynamics modeling.
\end{remark}

\section{Performance Evaluation}\label{VI}
In this section, we present numerical experiments to assess the performance and verify the effectiveness of the proposed MA-CDMP framework. We further examine key hyperparameters to characterize their influence and demonstrate robustness.

\subsection{Experimental Settings}
We implement the diffusion noise model $\bm{\epsilon}_{\theta}$ using a temporal U-Net\cite{ronneberger2015u} with $8$ residual blocks, organized into $4$ down-sampling layers and $4$ symmetric up-sampling layers. Timestep and condition embeddings are concatenated and injected into the first temporal convolution of each block. The classifier $\mathcal{J}_{\psi}$ adopts the down-sampling half of this U-Net, with a final linear layer to produce a scalar output. The inverse dynamics model $f_{\phi}$ is designed as a feed-forward neural network that maps observation transitions to actions. To enhance efficiency and stability, parameters are shared among homogeneous agents.

We employ OPNET as the simulation platform for the wireless communication network. The simulated area spans $10$ km $\times$ $10$ km with communication nodes randomly distributed. Wireless connectivity is configured through the radio parameters in OPNET, including transmit power, antenna gains, carrier frequency, and communication range. Each node supports 1-hop transmission within a $3.6$ km radius and operates over $4$ channels at a data rate of $2$ Mbps. The transmit power is $100$ W, and the antenna gains at both the transmitter and receiver are set to $1$. Each transmission frame consists of $10$ time slots, yielding $M \times L = 40$ RBs, from which each agent selects its RB demand with a per-agent action-space size of $41$. Each agent has a $4$-dimensional local observation and a $4$-dimensional MF observation, resulting in an $8$-dimensional per-step input after concatenation. To construct the offline dataset with diverse network loads and examine the generalization capability of offline resource allocation, the offline dataset is collected from $8$-node and $9$-node topologies, covering different high- to low-rate traffic ratios, i.e., $2:6$ and $4:4$ for $8$ nodes and $2:7$ and $4:5$ for $9$ nodes. Packet generation follows an exponential distribution, with the mean arrival time adjusted based on the speed ratio and network topology to ensure load balancing and prevent congestion in dense regions. RBs are allocated according to the high- to low-speed ratio in each scenario. We record the environment data per frame, producing a dataset of approximately \(6.0\times 10^{6}\) state-action-reward transition steps, where each transition stores the local observations, MF summaries, actions, and rewards of all agents. During testing, the learned models are directly deployed in OPNET without online fine-tuning and evaluated under random communication scenarios and larger node-number settings to examine topology-level generalization.

To assess the performance of the proposed MA-CDMP algorithm, we follow the evaluation protocol in \cite{shao2023counterfactual} and compare it with representative offline MARL baselines, including CFCQL\cite{shao2023counterfactual}, MA-TD3BC\cite{fujimoto2021minimalist}, MA-Diffuser\cite{janner2022planning}, MA-DT\cite{chen2021decision}, and MA-DD\cite{ajay2022conditional}. Among them, CFCQL is a CTDE-based offline MARL method, while the other baselines are adapted to decentralized execution for fair comparison. In addition, we compare with MADiff \cite{zhu2024madiff}, a representative diffusion-based MARL baseline that models coordinated multi-agent trajectories under the CTDE setting. All methods are trained on the same offline dataset for $100$ epochs, each epoch comprising $1{,}000$ training steps. Ablation studies use $909$ transition-step episodes to obtain stable estimates of module and hyperparameter effects, while comparison experiments use $181$ transition-step episodes for computational efficiency when plotting learning curves. Within each experiment, all methods share the same episode length, and performance is evaluated by the average reward and key QoS metrics, including average throughput, average delay, and packet loss rate. The key simulation parameters used throughout the experiments are listed in Table~\ref{tab:parameters}.

\begin{table}[!t]
    \caption{List of Key Parameter Settings for the Simulation.}
    \label{tab:parameters}
    \vspace{-0.2cm}
    \centering
    \renewcommand{\arraystretch}{1.2}
    \begin{tabularx}{\linewidth}{|X|p{3cm}|} 
        \hline
        \textbf{Parameters Description} & \textbf{Value} \\
        \hline
        % \multicolumn{2}{|l|}{\textbf{Simulator}} \\
        % \hline
        Number of time slots in a frame & $M = 10$ \\
        Number of channels & $L = 4$ \\
        Simulation area & $10$ km $\times$ $10$ km \\
        1-hop transmission radius of each node & $3.6$ km \\
        Packet transmission rate & $2$ Mbps \\
        % Channel bandwidth & $10$ kHz \\
        Transmit power & $100$ W \\
        Transmitter/receiver antenna gain & $1$ \\
        Offline dataset size & $6.0\times 10^{6}$ transition steps \\
        Comparison episode length & $181$ transition steps \\
        Ablation episode length & $909$ transition steps \\
        \hline
        % \multicolumn{2}{|l|}{\textbf{MA-CDMP}} \\
        % \hline
        % Diffusion steps & $K = 100$ \\
        % Planning horizon & $H = 8$ \\
        % Discount factor & $\gamma = 0.99$ \\
        % Conditional guidance scale & $\zeta = 1.2$ \\
        % \hline
    \end{tabularx}
\end{table}

\subsection{Numerical Results}
\subsubsection{Performance Comparison}
First, to demonstrate the superiority of the proposed MA-CDMP algorithm, we conduct comparative experiments, as shown in Fig.~\ref{MA-CDMP}. The results are averaged over $6$ independent random scenario seeds. It can be observed that MA-CDMP achieves higher average rewards after sufficient training and exhibits smoother convergence with smaller reward variance, indicating more stable learning behavior. These results validate the effectiveness of MA-CDMP, showing that the integration of conditional diffusion planning and MF interaction enables more reliable and consistent policy optimization in wireless communication scenarios.

\begin{figure}[t]
\centerline{\includegraphics[width=0.8\linewidth]{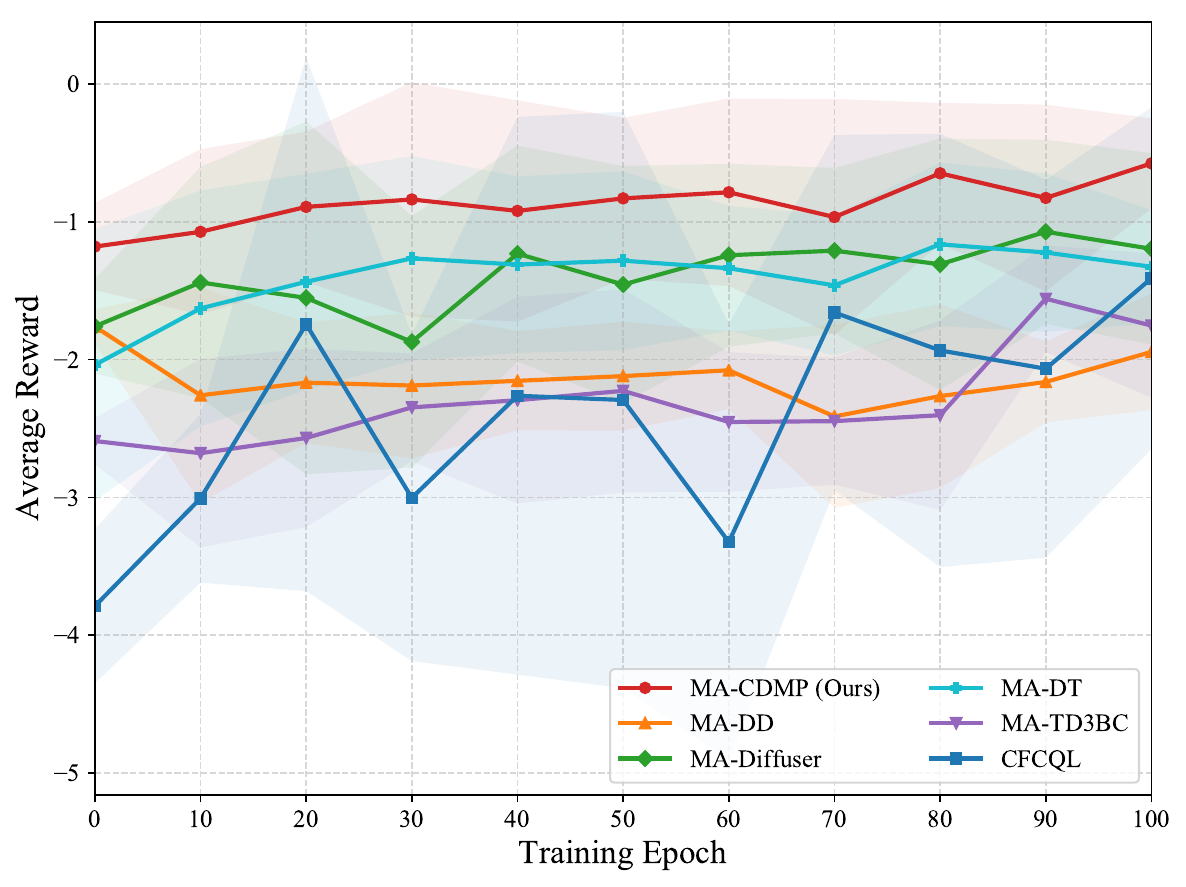}}
\vspace{-0.3cm}
\caption{Comparison of MA-CDMP with baselines in terms of average reward.}
\label{MA-CDMP}
\vspace{-0.4cm}
\end{figure}

\begin{table*}[!b]
\centering
\vspace{-.4cm}
\caption{QoS comparison of MA-CDMP with baselines. Results are reported as mean $\pm$ standard deviation over 6 independent seeds.}
\label{tab:qos_comparison}
\setlength{\tabcolsep}{10pt}
\renewcommand{\arraystretch}{1.08}
\small
\begin{tabular}{lccc}
\toprule
\textbf{Algorithms}
& \textbf{Packet Loss Rate (\%)} $\downarrow$
& \textbf{Average Throughput (kbps)} $\uparrow$
& \textbf{Average Delay (ms)} $\downarrow$ \\
\midrule
MA-DD
& $6.88 \pm 1.43$
& $304.91 \pm 18.32$
& $29.40 \pm 4.20$ \\

MA-Diffuser
& $4.40 \pm 1.69$
& $312.72 \pm 19.90$
& $22.00 \pm 6.90$ \\

MA-DT
& $4.54 \pm 0.91$
& $\mathbf{315.76 \pm 18.07}$
& $23.30 \pm 4.10$ \\

MA-TD3BC
& $6.72 \pm 1.42$
& $294.74 \pm 17.32$
& $27.50 \pm 5.20$ \\

CFCQL
& $6.25 \pm 3.06$
& $305.68 \pm 22.30$
& $24.10 \pm 12.40$ \\

\textbf{MA-CDMP}
& $\mathbf{2.29 \pm 0.99}$
& $309.82 \pm 18.29$
& $\mathbf{15.80 \pm 3.30}$ \\
\bottomrule
\end{tabular}
\end{table*}

Additionally, we conduct a separate comparison with MADiff~\cite{zhu2024madiff}, as shown in Fig.~\ref{DTDE}. Different from the offline baselines above, MADiff serves as a centralized diffusion-based reference for multi-agent trajectory generation. It can be observed that, although MA-CDMP operates under the stricter DTDE setting without access to global states or joint trajectories, it still achieves faster reward improvement and smaller performance variance. This result indicates that MF communication can effectively compensate for the absence of centralized information by capturing local interaction patterns, thereby supporting stable decentralized planning in dynamic wireless environments.

To further evaluate the effectiveness of MA-CDMP, we compare its detailed QoS performance with representative baselines, as shown in Table~\ref{tab:qos_comparison}. The results are reported as the mean and standard deviation over $6$ independent scenario seeds. MA-CDMP achieves the lowest packet loss rate and average delay, while maintaining competitive average throughput compared with the strongest baselines. This demonstrates that the proposed diffusion-based planner can generate effective resource-allocation decisions and improve overall QoS performance under decentralized multi-agent settings.

\begin{figure}[t]
\centerline{\includegraphics[width=0.8\linewidth]{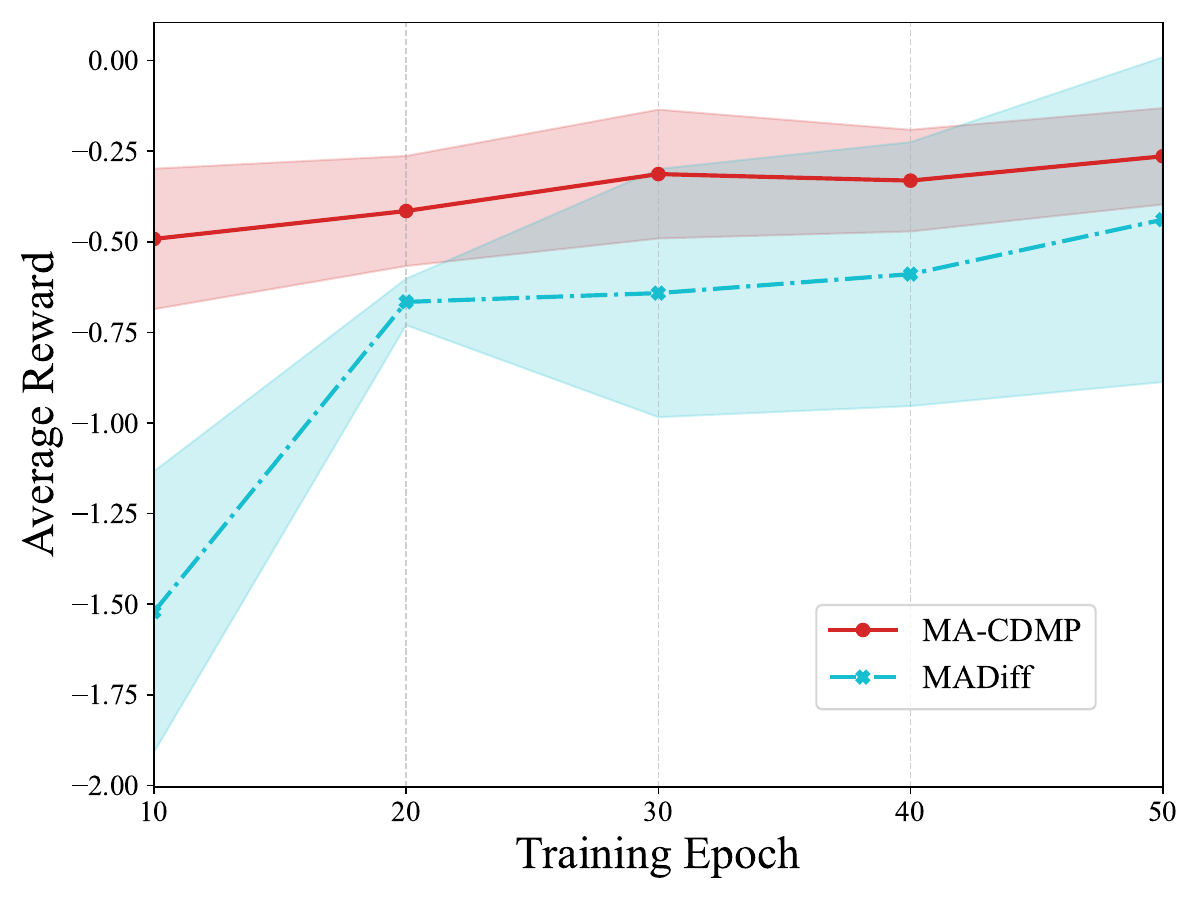}}
\vspace{-0.3cm}
\caption{Comparison of MA-CDMP with MADiff in terms of average reward.}
\label{DTDE}
\vspace{-0.4cm}
\end{figure}

\subsubsection{Ablation Studies}
To validate the effectiveness of the MF communication mechanism in the MA-CDMP framework, we design a variant that removes this component, denoted as MA-CDMP\_\textit{w/o\_MF}, and conduct comparative experiments under different node numbers. As illustrated in Fig.~\ref{scalability}, MA-CDMP consistently achieves higher average rewards than MA-CDMP\_\textit{w/o\_MF} across all settings. These results highlight the importance of the MF communication mechanism, which facilitates more effective cooperation among agents and enhances overall system performance in wireless environments. Furthermore, since the offline dataset is collected from $8$-node and $9$-node scenarios, the evaluations with $12$ and $15$ nodes represent unseen larger-scale topologies. The consistent advantage of MA-CDMP over MA-CDMP w/o MF in these settings demonstrates that the MF-enhanced diffusion planner can transfer the learned local interaction patterns to larger networks, confirming its strong scalability and generalization capability.

\begin{figure}[t]
\centerline{\includegraphics[width=\linewidth]{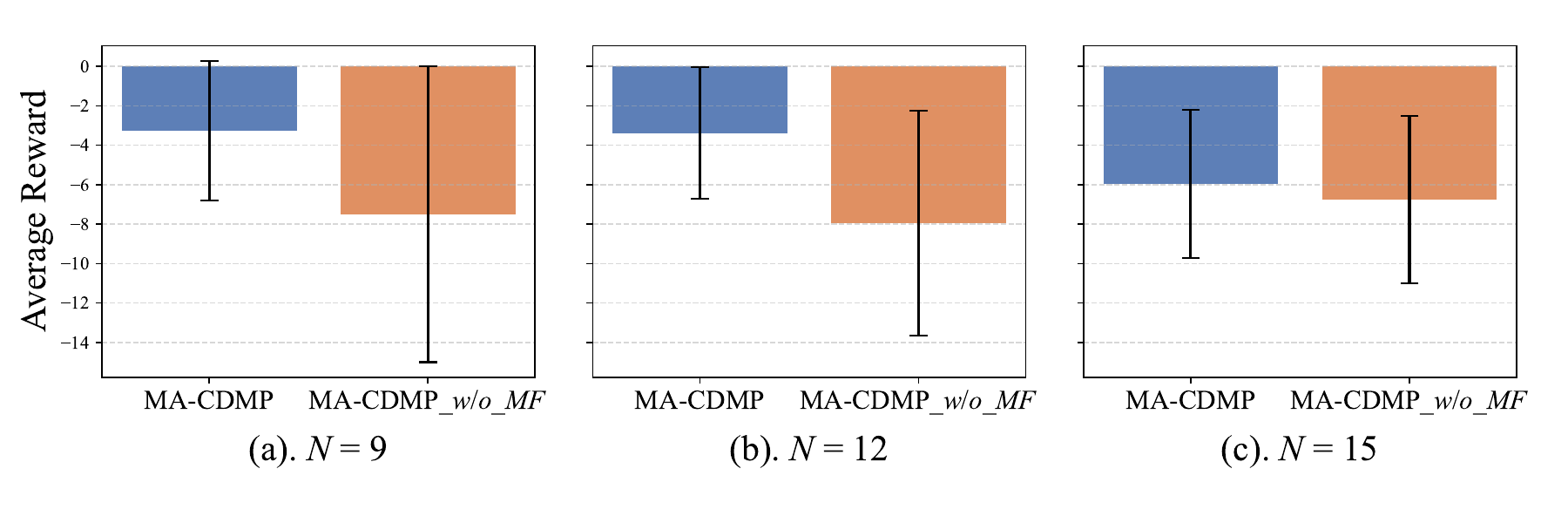}}
\vspace{-0.45cm}
\caption{Comparison of MA-CDMP with MA-CDMP\_\textit{w/o\_MF} under varying node number scenarios in terms of average reward.}
\label{scalability}
\vspace{-0.4cm}
\end{figure}

Next, to investigate the impact of different planning horizons $H$ on the performance of MA-CDMP, we conduct a comparative experiment as shown in Fig.~\ref{horizon}. The results indicate that extending the planning horizon initially improves performance, as a longer horizon enables agents to incorporate more temporal context and make decisions with broader foresight. However, when the horizon becomes excessively long, the performance declines. This degradation arises because the enlarged generative space increases the decision-making complexity and computational burden, thereby reducing the overall generation efficiency. Hence, in practical deployment, a balanced trade-off between computational cost and algorithmic performance should be carefully maintained.

\begin{figure}[t]
\centerline{\includegraphics[width=0.8\linewidth]{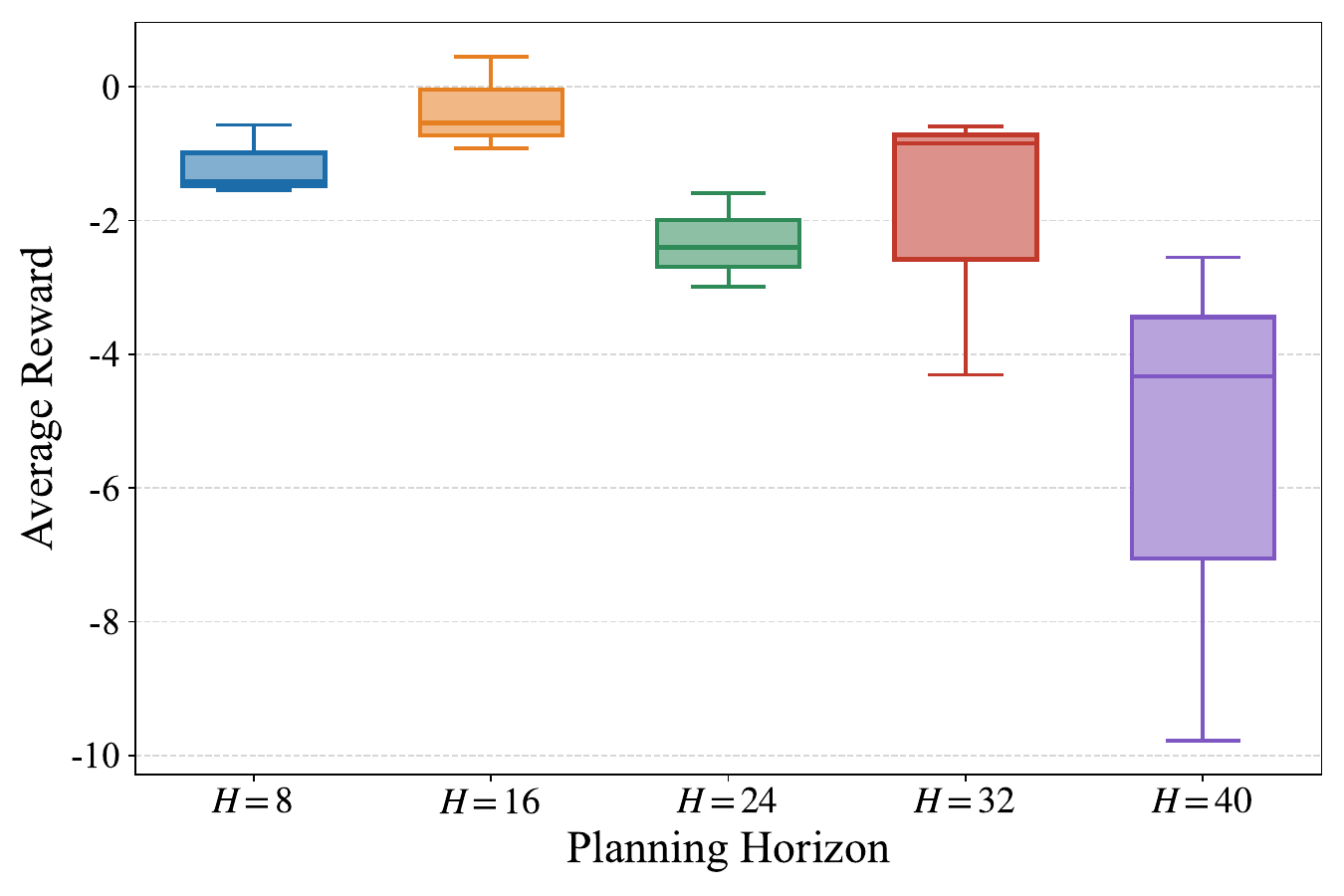}}
\vspace{-0.45cm}
\caption{Comparison of different planning horizons $H$ in terms of average reward.}
\label{horizon}
\vspace{-0.4cm}
\end{figure}

Furthermore, we investigate how the conditional guidance scale affects the conditional generation capability of the MA-CDMP algorithm. As illustrated in Fig.~\ref{scale}, when the value of $\zeta$ increases, the algorithm shows a modest improvement in average reward, yet the overall performance remains relatively consistent. This observation verifies Eq.~\eqref{eq:reverse mean}, suggesting that $\zeta$ serves as an effective control parameter for regulating the model’s ability to generate samples that meet specific conditional constraints. In addition, the stable trend across different $\zeta$ values demonstrates the robustness and adaptability of the model with respect to variations in conditional guidance scalar, ensuring reliable results under varying conditions.

\begin{figure}[t]
\centerline{\includegraphics[width=0.8\linewidth]{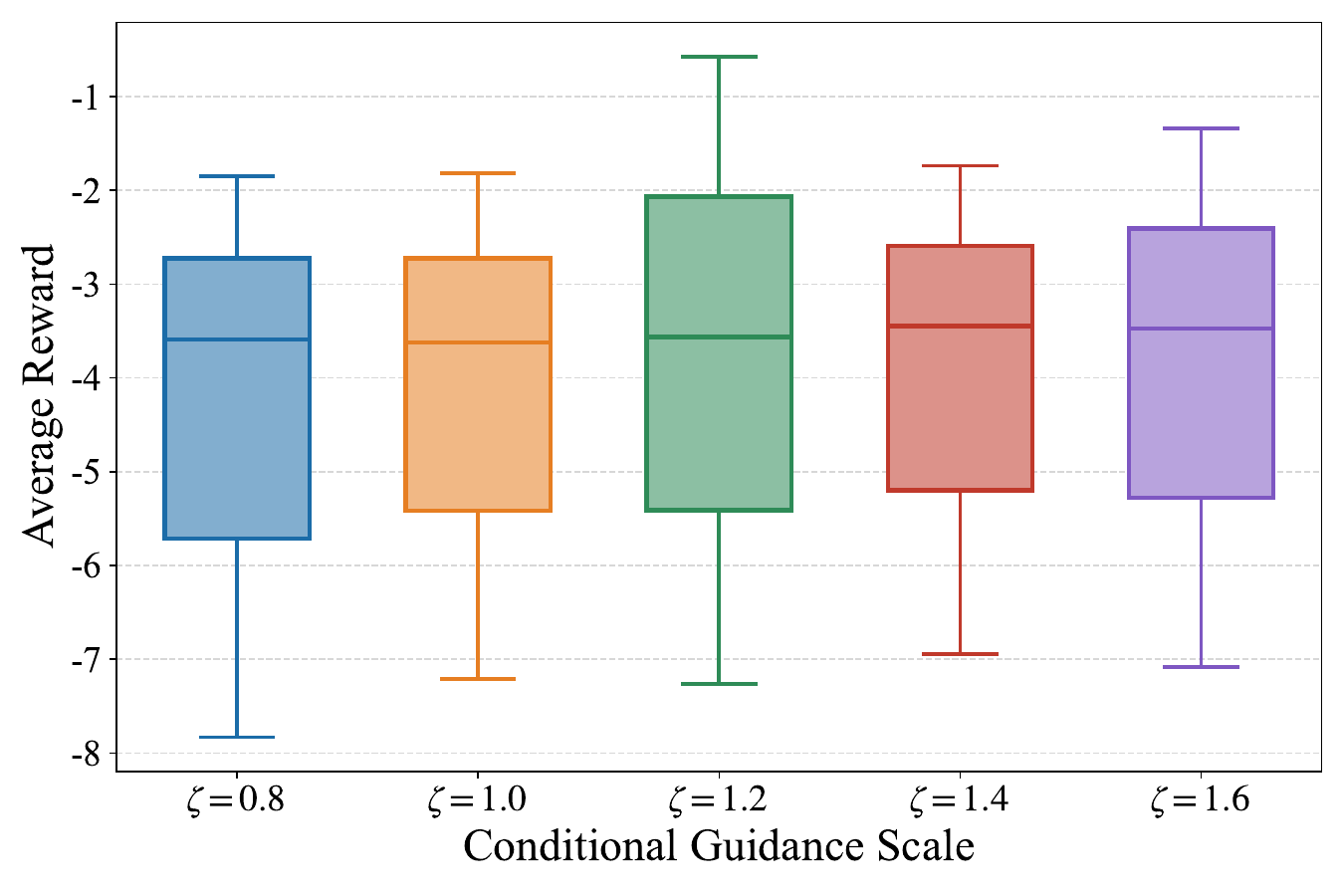}}
\vspace{-0.45cm}
\caption{Comparison of different conditional guidance scales $\zeta$ in terms of average reward.}
\label{scale}
\vspace{-0.4cm}
\end{figure}

We further integrate the DPM-Solver\cite{lu2022dpm} to examine the effect of diffusion step $K$ on the efficiency of decision generation in MA-CDMP. Using a pre-trained noise model with $K=100$ steps, we apply the first-order DPM-Solver to reduce sampling iterations and evaluate its performance.
% under $K=80, 60, 40$, and $20$. 
The results in Fig.~\ref{K} show that although reducing diffusion steps leads to a slight decline in performance, MA-CDMP maintains stable average rewards and achieves performance comparable to baseline methods. Moreover, fewer sampling steps significantly lower the computational and time costs, enhancing overall decision-making efficiency. In practical deployments, $K$ can be adjusted to balance performance and computational demand according to specific deployment requirements.

\begin{figure}[t]
\centerline{\includegraphics[width=0.8\linewidth]{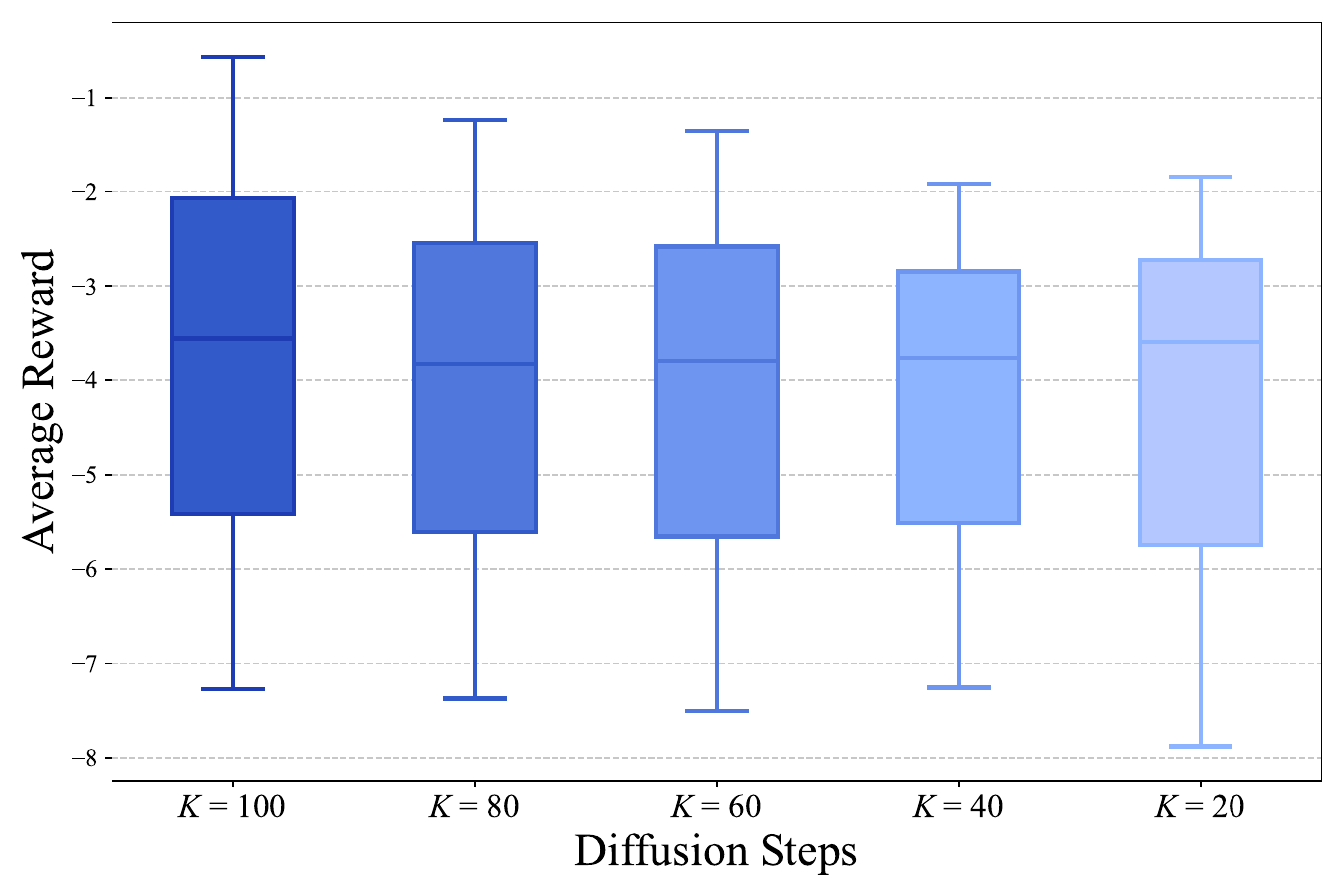}}
\vspace{-0.45cm}
\caption{Comparison of different diffusion steps $K$ in terms of average reward.}
\label{K}
\vspace{-0.4cm}
\end{figure}

\section{Conclusion}\label{VII}
In this paper, we have investigated decentralized resource allocation in multi-node, resource-constrained wireless communication networks. Under the MBRL paradigm, we have proposed the MA-CDMP algorithm, which significantly improves communication efficiency and system performance. Specifically, we have employed a conditional diffusion model to predict local observation sequences with high cumulative rewards and utilized an inverse dynamics model to generate corresponding actions. To enable cooperative behavior among distributed nodes and optimize global performance, we have integrated MF communication to enhance the performance of classifier-based world modeling. Moreover, we have theoretically established an upper bound on the discrepancy between the diffusion-generated distribution under MF approximation and the true distribution. Extensive simulations have confirmed the effectiveness, robustness, and scalability of the proposed algorithm.

% For future work, several promising directions remain to be explored. For example, the iterative denoising process in diffusion generation incurs notable computational and time costs during decision-making. Although the adoption of the first-order DPM-Solver in this paper improves efficiency, a minor performance gap still exists. Future research may explore consistency models or flow-based models to further accelerate diffusion generation while preserving accuracy. In addition, while MF communication enables efficient information exchange with low overhead, it inevitably introduces approximation errors. Integrating attention mechanisms into the diffusion backbone may help mitigate these limitations, leading to more coordinated and reliable decision-making in large-scale multi-agent systems.

\appendices
\section{Proof for Lemma \ref{lemma:OU process}}\label{proof:lemma 2}
\begin{proof}
Based on It\^{o}'s Lemma \cite{ito1951stochastic}, the first-order linear SDE with time-varying coefficients admits the closed-form solution as
\begin{equation}
\label{eq:closed-form}
    p_{T}(\bm{x}_{T}| \bm{x}_{{\tau_0}}) = \mathcal{N}\!\left( e^{-\tfrac{1}{2}\bar{\beta}_T} \bm{x}_{{\tau_0}}, \; \big(1 - e^{-\bar{\beta}_T}\big) \bm{I} \right).
\end{equation}
To bound the divergence between $p_T$ and the standard Gaussian $\rho$, we apply Jensen’s inequality, which yields
\begin{align}\label{eq:Jesen's inequality}
    D_{\text{KL}}(p_{T}\,\|\,\rho) &= D_{\text{KL}}\!\left( \int p_{T}(\bm{x}_{T} | \bm{x}_{{\tau_0}}) \, p_{{\tau_0}}(\bm{x}_{{\tau_0}})\mathrm{d}\bm{x}_{{\tau_0}} \,\Big\|\, \rho \right) \nonumber \\
&\leq \mathbb{E}_{\bm{x}_{{\tau_0}}\sim p_{{\tau_0}}} \!\left[ D_{\text{KL}}\!\left(p_{T}(\bm{x}_{T} | \bm{x}_{{\tau_0}}) \,\|\, \rho \right) \right].
\end{align}
For a $d$-dimensional Gaussian distribution $\mathcal{N}(\bm{\mu}, \bm{\Sigma})$, the KL divergence with respect to $\rho$ has a closed form as 
\begin{equation}
    D_{\text{KL}}\!\left( \mathcal{N}(\bm{\mu}, \!\bm{\Sigma}) \| \rho \right) \!=\! \tfrac{1}{2} \!\left( \!\text{Tr}(\bm{\Sigma})\! -\! d \!-\! \log\!\det(\bm{\Sigma})\! +\! \left\| \bm{\mu}\right\|^{2} \!\right).
\end{equation}
Substituting the mean and covariance from Eq.~\eqref{eq:closed-form}, we obtain
\begin{equation}
\scalebox{0.95}{$
    D_{\text{KL}}\!\left(p_{T}\|\rho \right) \!=\! \tfrac{1}{2}\! \left[ \!-d \!\log\!\left(\!1 \!-\! e^{\!-\bar{\beta}_T} \!\right) \!-\! d e^{\!-\bar{\beta}_T} \!+\! e^{\!-\bar{\beta}_T} \!\left\| \bm{x}_{{\tau_0}} \right\|\!^{2} \right].
$}
\end{equation}
Taking expectation over $\bm{x}_{\tau_0} \sim p_{\tau_0}$ and using the second-order moment $M_2$, Eq.~\eqref{eq:Jesen's inequality} leads to
\begin{align}
D_{\text{KL}}\!\left(p_{T}\| \rho\right) 
&\scalebox{0.9}{$\leq \mathbb{E}_{p_{{\tau_0}}}\!\left[\tfrac{1}{2}\!\left(\!-d\log(1\!-\!e^{\!-\bar{\beta}_T})\!-\!d e^{\!-\bar{\beta}_T}\!+\!e^{\!-\bar{\beta}_T}\|\bm{x}_{{\tau_0}}\|^{2}\right)\right]$} \nonumber \\
&\!= \tfrac{1}{2}\!\left[\!-d\log(1\!-\!e^{\!-\bar{\beta}_T})\!-\!d e^{\!-\bar{\beta}_T}\!+\!M_{2} e^{\!-\bar{\beta}_T}\right]\!.
\end{align}
Finally, under appropriately designed noise schedule, $e^{-\bar{\beta}_T}$ becomes sufficiently small. Applying the first-order approximation $\log(1 - e^{-\bar{\beta}_T}) \approx -e^{-\bar{\beta}_T}$, we derive
\begin{equation}
    D_{\text{KL}}\!\left(p_{T}\,\|\,\rho\right) 
\lesssim \tfrac{1}{2} M_{2}\, e^{-\bar{\beta}_T}.
\end{equation}
Then the lemma comes.
\end{proof}

\section{Proof for Lemma \ref{lemma:relative Fisher information}} \label{proof:lemma 3}
\begin{proof}
By the Fokker-Planck equation\cite{risken1989fokker}, the evolution of $p_{\tau}$ and $q_{\tau}$ is given by
\begin{align}
\frac{\partial p_{\tau}}{\partial \tau}(X_{\tau}) \!&=\! \nabla \left[ \!-\!F_1(X_{\tau})p_{\tau}(X_{\tau}) \!+\! \tfrac{g(\tau)^2}{2}\nabla p_{\tau}(X_{\tau}) \right], \\
\frac{\partial q_{\tau}}{\partial \tau}(X_{\tau}) \!&=\! \nabla \left[ \!-\!F_2(X_{\tau})q_{\tau}(X_{\tau}) \!+\! \tfrac{g(\tau)^2}{2}\nabla q_{\tau}(X_{\tau}) \right]. 
\end{align}
Consequently, we can write
\begin{equation}
    \frac{\partial}{\partial \tau} D_{\text{KL}}(p_{\tau}\|q_{\tau})
\!=\! \int \log \frac{p_{\tau}}{q_{\tau}} \frac{\partial p_{\tau}}{\partial \tau}  \mathrm{d}X_{\tau} \!-\! \int \frac{p_{\tau}}{q_{\tau}} \frac{\partial q_{\tau}}{\partial \tau}  \mathrm{d}X_{\tau}.
\end{equation}
% \begin{equation}
%     \frac{\partial}{\partial \tau} D_{\text{KL}}(p_{\tau}\|q_{\tau})
% \!=\! \int \left(\log \frac{p_{\tau}}{q_{\tau}} +1\right) \frac{\partial p_{\tau}}{\partial \tau}  \mathrm{d}X_{\tau} \!-\! \int \frac{p_{\tau}}{q_{\tau}} \frac{\partial q_{\tau}}{\partial \tau}  \mathrm{d}X_{\tau}.
% \end{equation}

For the first term, substituting $\tfrac{\partial p_{\tau}}{\partial \tau}$ and applying integration by parts under the assumption that the boundary term vanishes (i.e., the probability flux at infinity is zero), we obtain
\begin{align}
\int &\log \frac{p_{\tau}}{q_{\tau}} \frac{\partial p_{\tau}}{\partial \tau} \mathrm{d}X_{\tau}
\!=\! \int\!\nabla \Big[ \!-\!F_1p_{\tau}\!+\! \tfrac{g(\tau)^2}{2} \!\nabla p_{\tau} \Big]\! \log \!\frac{p_{\tau}}{q_{\tau}} \mathrm{d}X_{\tau} \nonumber\\
&\!=\! \int\Big\langle \nabla \log \tfrac{p_{\tau}}{q_{\tau}} ,  F_1p_{\tau} - \tfrac{g(\tau)^2}{2} \nabla p_{\tau} \Big\rangle\, \mathrm{d}X_{\tau} \nonumber\\
&\scalebox{0.9}{$=\! \int \!p_{\tau} \Big\langle F_1, \nabla \log \tfrac{p_{\tau}}{q_{\tau}} \Big\rangle \, \mathrm{d}X_{\tau}
 \!-\! \int \tfrac{g(\tau)^2}{2} \!\Big\langle \nabla \log \tfrac{p_{\tau}}{q_{\tau}}, \nabla p_{\tau} \!\Big\rangle \, \mathrm{d}X_{\tau}.$}
\end{align}
Similarly, the second term yields
\begin{align}
\int \frac{p_{\tau}}{q_{\tau}} &\frac{\partial q_{\tau}}{\partial \tau} \, \mathrm{d}X_{\tau}
\!=\! \!\int \frac{p_{\tau}}{q_{\tau}} \nabla \Big[ \!-\!F_2q_{\tau} +\tfrac{g(\tau)^2}{2} \nabla q_{\tau} \Big] \mathrm{d}X_{\tau} \nonumber\\
&\!= \int \Big\langle \nabla \tfrac{p_{\tau}}{q_{\tau}}, F_{2} q_{\tau} - \tfrac{g(\tau)^2}{2} \nabla q_{\tau} \Big\rangle \, \mathrm{d}X_{\tau} \nonumber\\
&\scalebox{0.99}{$= \!\int \!q_{\tau} \Big\langle \nabla \tfrac{p_{\tau}}{q_{\tau}}, F_2 \Big\rangle  \mathrm{d}X_{\tau}
 \!-\! \tfrac{g(\tau)^2}{2} \!\int \!\Big\langle \nabla \tfrac{p_{\tau}}{q_{\tau}}, \nabla q_{\tau} \Big\rangle  \mathrm{d}X_{\tau}.$}
\end{align}
Notice that
\begin{align}
&\scalebox{0.95}{$\int \left\langle \nabla \frac{p_{\tau}}{q_{\tau}}, \nabla q_{\tau} \right\rangle \mathrm{d}X_{\tau} \!-\! \int \left\langle \nabla \log \frac{p_{\tau}}{q_{\tau}}, \nabla p_{\tau} \right\rangle \mathrm{d}X_{\tau} \nonumber $}\\
&\scalebox{0.89}{$=\! \int\! \left\langle \frac{q_{\tau} \nabla p_{\tau} - p_{\tau} \nabla q_{\tau}}{q_{\tau}}, \nabla \log q_{\tau} \right\rangle \mathrm{d}X_{\tau} \!-\! \int p_{\tau} \left\langle \nabla \log \frac{p_{\tau}}{q_{\tau}}, \nabla \log p_{\tau} \right\rangle \mathrm{d}X_{\tau} $}\nonumber \\
&\scalebox{0.9}{$= \int p_{\tau} \left\langle \nabla \log \frac{p_{\tau}}{q_{\tau}}, \nabla \log q_{\tau} \right\rangle \mathrm{d}X_{\tau} 
- \int p_{\tau} \left\langle \nabla \log \frac{p_{\tau}}{q_{\tau}}, \nabla \log p_{\tau} \right\rangle \mathrm{d}X_{\tau} $}\nonumber \\
&\!= - J(p_{\tau}\|q_{\tau}),
\end{align}
and
\begin{align}
&\int p_{\tau} \left\langle F_{1}, \nabla \log \frac{p_{\tau}}{q_{\tau}} \right\rangle  \mathrm{d}X_{\tau} \!-\! \int q_{\tau} 
\left\langle \nabla \frac{p_{\tau}}{q_{\tau}}, F_{2} \right\rangle \mathrm{d}X_{\tau} \nonumber \\
&=\! \!\int\! p_{\tau}\!\left\langle \!F_{1}, \nabla\! \log \frac{p_{\tau}}{q_{\tau}}\! \right\rangle  \mathrm{d}X_{\tau} \!-\! \int\! p_{\tau} \!\left\langle \!\nabla\! \log \!\frac{p_{\tau}}{q_{\tau}}, F_2 \!\right\rangle \mathrm{d}X_{\tau} \nonumber \\
&= \int p_{\tau} \left\langle \nabla \log \frac{p_{\tau}}{q_{\tau}}, F_{1} \!-\! F_{2} \right\rangle \, \mathrm{d}X_{\tau} \nonumber \\
&= \mathbb{E}\!\left[ 
\left\langle F_{1} - F_{2}, \nabla \log \frac{p_{\tau}}{q_{\tau}} \right\rangle 
\right].
\end{align}
Therefore, the lemma follows.
\end{proof}

% \bibliographystyle{ieeetran}
% \bibliography{hyperref}


 

% Use the author name as the 3rd argument, followed by the biography text.
% \end{IEEEbiography}


\end{document}